\pdfoutput=1
\documentclass{article} %
\usepackage[dvipsnames]{xcolor}
\usepackage{colm2024_conference_arxiv}

\usepackage{microtype}
\usepackage{hyperref}
\usepackage{url}
\usepackage{booktabs}
\usepackage{amsmath}
\usepackage{graphicx}
\usepackage{listings}
\usepackage{epigraph} 
\usepackage{tabularx}

\usepackage{xcolor}
\usepackage{amsmath,amsfonts,bm,amssymb,amsthm}

\DeclareMathOperator{\defeq}{\stackrel{\text{def}}{\;=\;}}

\newtheorem{theorem}{Theorem}
\newtheorem{lemma}[theorem]{Lemma}

\def\eqref#1{equation~\ref{#1}}

\def\1{\bm{1}}

\DeclareMathAlphabet{\mathsfit}{\encodingdefault}{\sfdefault}{m}{sl}
\SetMathAlphabet{\mathsfit}{bold}{\encodingdefault}{\sfdefault}{bx}{n}

\newcommand{\E}{\mathbb{E}}

\newcommand{\R}{\mathbb{R}}

\definecolor{darkblue}{rgb}{0, 0, 0.5}
\hypersetup{colorlinks=true, citecolor=darkblue, linkcolor=darkblue, urlcolor=darkblue}
\renewcommand{\R}{\mathbb R}

\newcommand{\replace}[1]{#1}
\newcommand{\replacemultiple}[1]{#1}
\newcommand{\accumulate}[1]{#1}

\title{Is Model Collapse Inevitable? Breaking the Curse of\\ Recursion by Accumulating Real and Synthetic Data}

\author{Matthias Gerstgrasser\thanks{Denotes equal authorship.}\, \thanks{Harvard \& Stanford University.} \,, Rylan Schaeffer$^*$, Apratim Dey$^*$, Rafael Rafailov$^*$, Dhruv Pai \\
Stanford University \\
\texttt{\{mgerst,rschaef,apd1995,rafailov,dhruvpai\}@stanford.edu} \\
\AND
Henry Sleight\thanks{Denotes equal contribution.}\,, John Hughes$^{\ddag}$, Tomasz Korbak$^{\ddag}$, Rajashree Agrawal$^{\ddag}$\\
Constellation \\
\AND
Andrey Gromov\\
University of Maryland, College Park \\
\AND
Daniel A. Roberts\\
MIT \& Sequoia Capital\\
\AND
Diyi Yang, David Donoho \& Sanmi Koyejo\\
Stanford University\\
\texttt{\{diyiy,donoho,sanmi\}@stanford.edu} \\
}

\colmfinalcopy %
\begin{document}

\maketitle

\begin{abstract}
The proliferation of generative models, combined with pretraining on web-scale data, raises a timely question: what happens when these models are trained on their own generated outputs? Recent investigations into model-data feedback loops proposed that such loops would lead to a phenomenon termed \textit{model collapse}, under which performance progressively degrades with each model-data feedback iteration until fitted models become useless. However, those studies largely assumed that new data \textit{\replace{replace}} old data over time, where an arguably more realistic assumption is that data \textit{\accumulate{accumulate}} over time. In this paper, we ask: what effect does accumulating data have on model collapse? 
We empirically study this question by pretraining sequences of language models on text corpora. We confirm that \replace{replacing the original real data by each generation's synthetic data does indeed tend towards model collapse}, then demonstrate that \accumulate{accumulating the successive generations of synthetic data alongside the original real data avoids model collapse}; these results hold across a range of model sizes, architectures, and hyperparameters. We obtain similar results for deep generative models on other types of real data: diffusion models for molecule conformation generation and variational autoencoders for image generation. To understand why accumulating data can avoid model collapse, we use an analytically tractable framework introduced by prior work in which a sequence of linear models are fit to the previous models' outputs. Previous work used this framework to show that \replace{if data are replaced, the test error increases with the number of model-fitting iterations}; we extend this argument to prove that \accumulate{if data instead accumulate}, the test error has a finite upper bound independent of the number of iterations, meaning \accumulate{model collapse no longer occurs}. 
Our work provides consistent empirical and theoretical evidence that \accumulate{data accumulation avoids model collapse}.
\end{abstract}

\section{Introduction}
\label{sec:introduction}

The advent of large-scale generative models such as GPT-4 \citep{openai2023gpt4}, DALL-E \citep{ramesh2022hierarchical} and Stable Diffusion \citep{rombach2022high} has revolutionized the field of artificial intelligence. These models, trained on vast web-scale datasets, exhibit remarkable capabilities in generating text, images, and other media \citep{brown2020language, saharia2022photorealistic}. However, as these models become more widely used, an increasing amount of generated data populates the web. This raises a critical question: what are the consequences of training generative models on datasets containing their own outputs?

Recent studies have investigated this question, revealing that training generative models on their own outputs can cause the performance of such models to progressively degrade with each model-fitting iteration, eventually rendering newer models useless \citep{hataya2023will,martinez2023combining,shumailov2023curse, alemohammad2023self,martinez2023towards,bertrand2023stability,briesch2023large, dohmatob2024model, dohmatob2024tale} (see Appendix \ref{app:sec:prior_work} for review and discussion of prior work). This phenomenon was consequently labeled \textit{model collapse}. Model collapse warns that democratizing access to generative models runs the risk of polluting the very data necessary to train future iterations of generative models.

\begin{figure}[t]
    \centering
    \includegraphics[width=\textwidth]{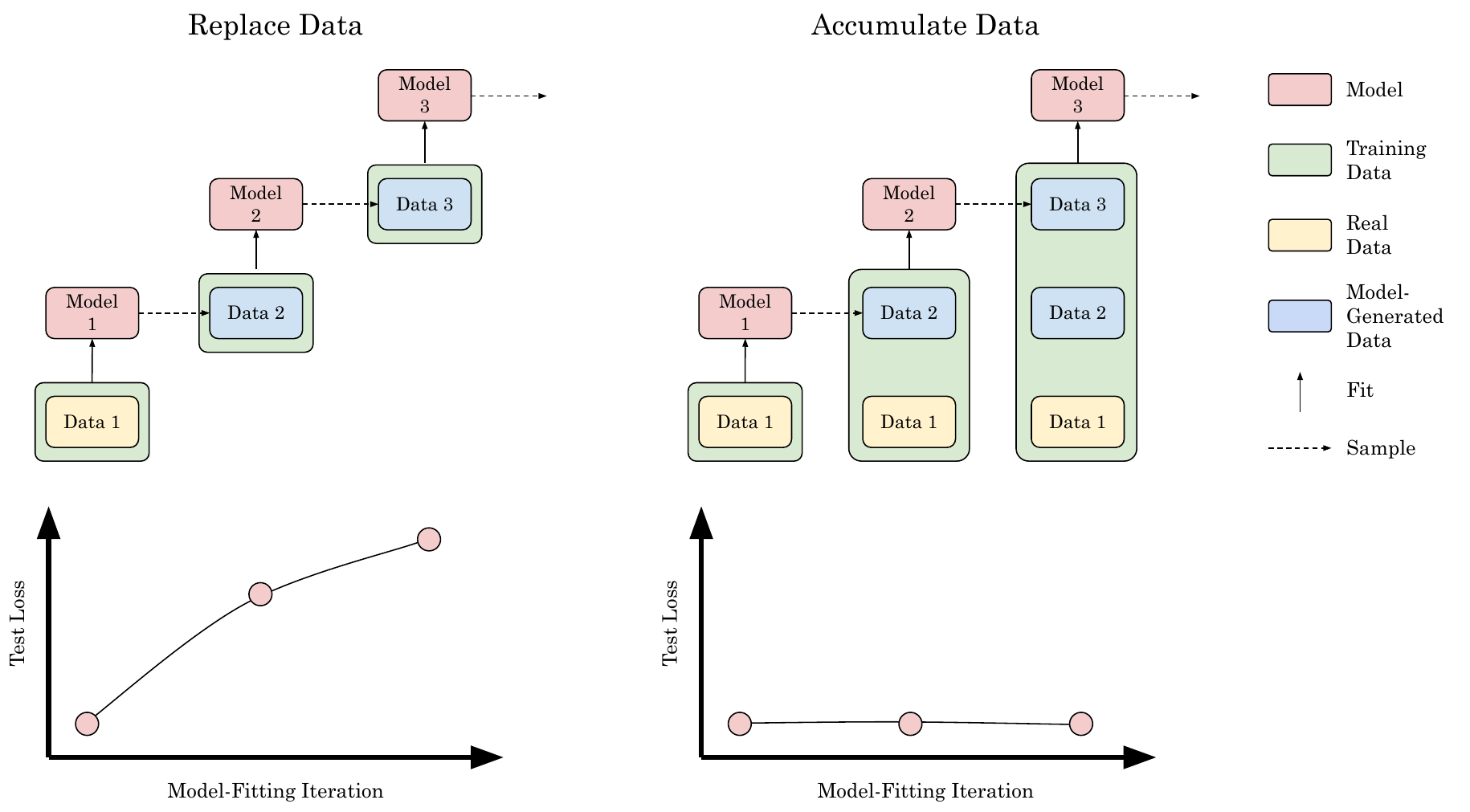}
    \caption{\textbf{Two Settings to Study Model Collapse.} Model collapse is a phenomenon where sequences of generative models trained on their own outputs progressively degrade until the latest model becomes useless. Left: Many prior works studied model collapse where data are \textit{\replace{replaced}} with each model-fitting iteration. Right: We study model collapse where data \textit{\accumulate{accumulate}} with each iteration and demonstrate accumulating data avoids model collapse.}
    \label{fig:schematic}
\end{figure}

To better understand this phenomenon
many prior works have considered a setup
that assumes each model's generated data \textit{\replace{replaces}} previous data.  %
In theory,
this leads to very natural comparisons across generations as the total number of training points for each model remains fixed.
In practice,
subsequent generations of LLMs are often trained with
increasing data over time -- e.g., 1.4 trillion tokens for Llama 1 \citep{touvron2023llama1}, 2 trillion for Llama 2 \citep{touvron2023llama2}, 15 trillion for Llama 3 -- 
in which presumably both 
human-generated and machine-generated data are accumulating in training sets collected from the internet. It was noted in some of those works \cite{hataya2023will, martinez2023combining, alemohammad2023self, bertrand2023stability, dohmatob2024tale} that model collapse can be either slowed down or negated by mixing in clean data with the generated data.

To that end, in this work we study the effect of \textit{\accumulate{accumulating}} data on model collapse, rather than replacing data. Our data-accumulating setting is, in some sense, maximally pessimistic: it considers a hypothetical future where synthetic data are uncontrollably dumped on the internet to be vacuumed up for training the next iteration of generative models. Nevertheless, we find that \textit{\accumulate{model collapse is avoided when accumulating data}}. 

We begin by studying model collapse experimentally with deep generative models trained on realistic data: transformers on causal language modeling (Sec. \ref{sec:real:subsec:language_transformers}), diffusion models on molecular conformation (Sec. \ref{sec:real:subsec:molecular_diffusion}) and variational autoencoders on images (Sec. \ref{sec:real:subsec:image_vaes}).
After confirming that \replace{replacing data at every iteration indeed causes test error to increase with the number of iterations}, we empirically find that \accumulate{accumulating synthetic data with real data avoids model collapse for all models and for all data modalities we test}.
To understand why replacing data and accumulating data have different consequences for model collapse, we turn to an analytically tractable framework of a sequence of linear models, each trained on synthetic outputs generated from the previous-iteration's fitted linear model \citep{mobahi2020self, dohmatob2024model}. Within this framework, \citet{dohmatob2024model} demonstrated that \replace{if data are replaced with each model-fitting iteration, the test error increases linearly} with the number of iterations $n$. We extend \citet{dohmatob2024model}'s analysis to prove that if \accumulate{data instead accumulate, then the test error has a finite and (to us, surprisingly) well-controlled upper bound independent of the number of model-fitting iterations}.\footnote{An approach `halfway' between the `replace' and `accumulate' settings replaces the previous dataset with a pure synthetic dataset of size $T\times i$ at the $i$-th iteration; in other words,
the comparison at each generation is now between models trained on the same number of training points.
While this \emph{halfway} setting has a milder sublinear behavior -- explicitly, the test MSE scaling is $MSE \asymp O(\log(n))$ --  
the test error does still diverge with iterations. (See Appendix~\ref{app:sec:linear_reg_replace_multiple} for more details.)
We thank Elvis Dohmatob, Yunzhen Feng, and Julia Kempe for communicating this observation.  We also consider the halfway setting as an ablation in our language modeling experiments, as detailed in Appendix~\ref{app:sec:lm_ablations}.}

Altogether, our work suggests that data accumulation may be robust to model collapse and emphasizes the importance of considering accumulating data and other real-world data dynamics in the analysis of model collapse in generative models trained on web-scale data.

\section{Accumulating Data Avoids Model Collapse in Deep\newline Generative Models}
\label{sec:realistic_generative_models}

We first investigate model collapse experimentally in several classes of generative models.
Here, and for the remainder of this manuscript, the term  {\em model collapse} refers to notably worsening error over increasing iterations of the model-data loop, while {\em avoiding model collapse} refers instead to bounded error over such iterations.
To test the effect of accumulating data on model collapse, we compare accumulating data against replacing data. We use three diverse experimental setups of causal transformers, diffusion models, and variational autoencoders trained on real text, molecular conformation, and image datasets, respectively. We find that replacing data yields model collapse for all models and all datasets, whereas accumulating data avoids model collapse.

\begin{figure}
    \centering
    \includegraphics[width=\textwidth]{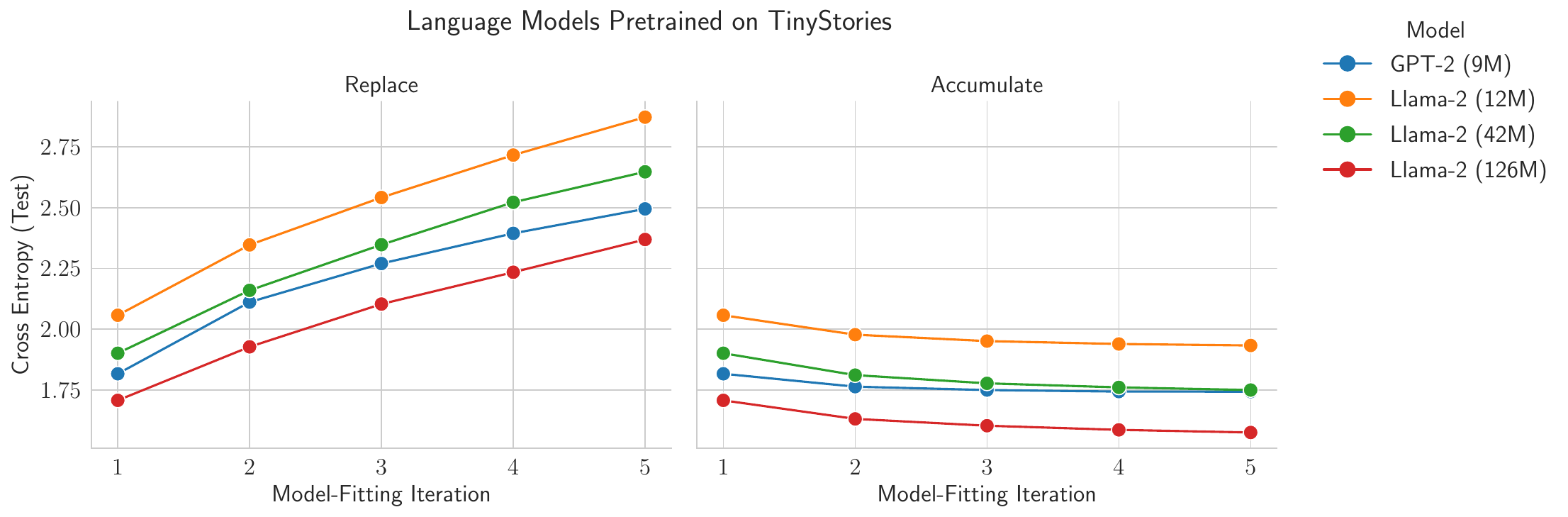}
    \caption{\textbf{Data Accumulation Avoids Model Collapse in Language Modeling.} Sequences of causal transformer-based language models are pretrained on TinyStories \citep{eldan2023tinystories}. Cross-entropy validation loss increases when repeatedly \replace{replacing} data (left), but not when \accumulate{accumulating} data (right). Synthetic data was sampled with temperature $= 1.0$.
    }
    \label{fig:language_modeling_results}
\end{figure}

\subsection{Transformer-Based Causal Language Modeling}
\label{sec:real:subsec:language_transformers}

\paragraph{Experiments} We first train causal transformers \citep{vaswani2017attention} on text data. Specifically, we pretrain 9M parameter GPT-2 \citep{radford2019language} and 12M, 42M and 125M parameter Llama2 \citep{touvron2023llama2} language models for a single epoch on TinyStories \citep{eldan2023tinystories}, a 470M token GPT-3.5/4-generated dataset of short stories at a kindergarten reading level. For each model-fitting iteration $n \geq 2$, we sample a new dataset of the same size as TinyStories from the previous iteration's language model and then either replace or concatenate the previous dataset with the newly generated dataset. In each model-fitting iteration, we then pretrain a newly initialized model on the replaced or concatenated dataset from the previous iteration. We experiment with sampling the new datasets using temperatures $0.3$ or $1.0$.  We chose this combination of architectures, scales, dataset, and sampling because the setup necessitates pretraining multiple iterations of language models -- a computationally costly endeavor -- but we also wish to study realistic conditions where generative models are high-performing and generative diverse outputs. Because small language models (below 10M parameters) pretrained on TinyStories were shown to be able to generate coherent-albeit-simple English sentences \citep{eldan2023tinystories}, this choice of architectures, scales, dataset and temperature hopefully strikes a good balance between being representative, being diverse and being computationally feasible.

\begin{figure}
    \centering
    \includegraphics[width=\textwidth]{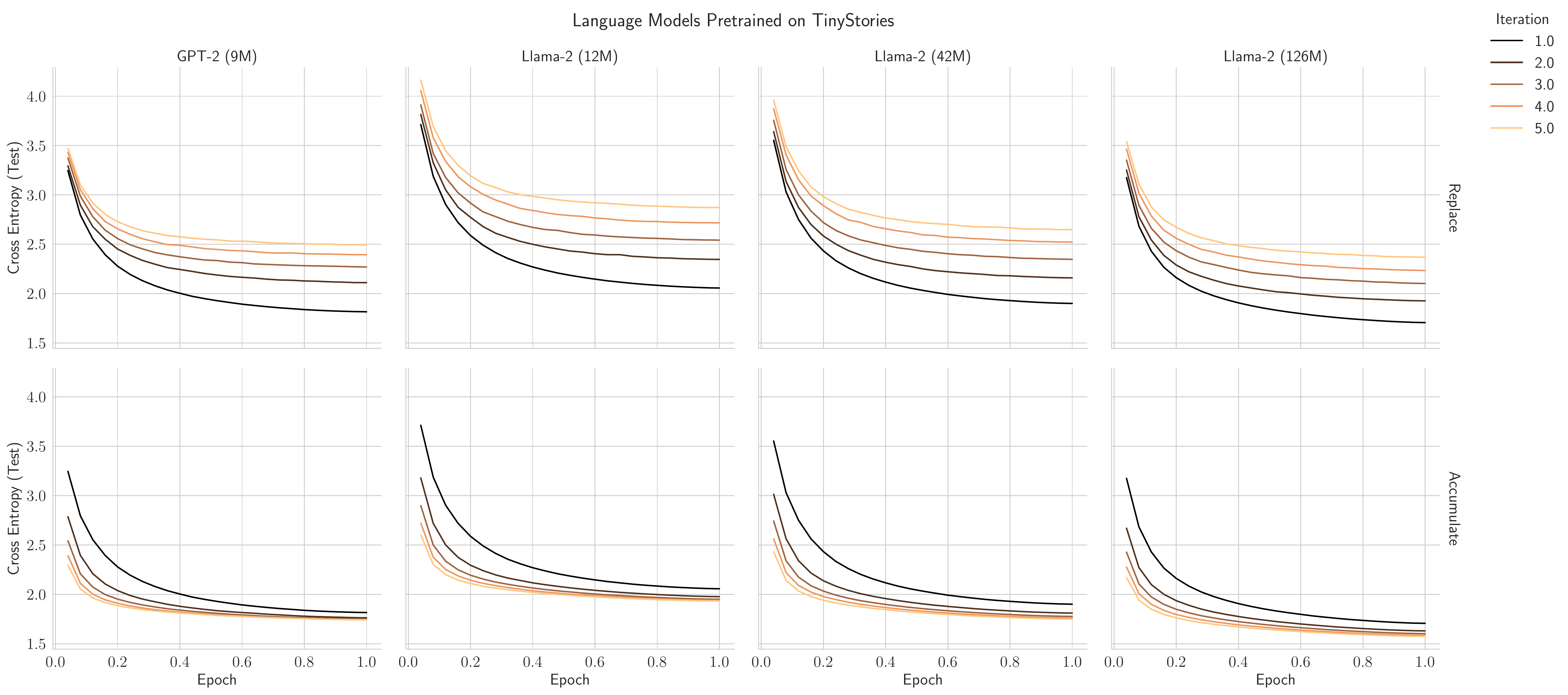}
    \caption{\textbf{Data Accumulation Avoids Model Collapse in Language Modeling.} Learning curves for individual model-fitting iterations when repeatedly \textit{replacing} data (top), and when \textit{accumulating} data (bottom). Note: Epochs correspond to more gradient steps for accumulate than replace because the number of training data grows for accumulate.
    }
    \label{fig:language_modeling_results_learning_curves}
\end{figure}

\begin{table}[h!]
    \centering
    \small
    \begin{tabularx}{\textwidth}{c|c|X}
        Model & Iteration & Sample Generation \\
        \hline 
        Llama2 (125M) & 3 \accumulate{(A)} & In the end, the crab found a smooth shell. He took it to a safe place under a tree. The crab put the shell where he found it. Tim and his mom were tired, but they were happy. They had a fun day at the beach. And they lived happily ever after. The end. \\
                    & 3 \replace{(R)} & Henry asked his Mom why the golf sounded so special. His Mom explained that the line of lumber had something special that would help. She said that if you're not sure, the lumber is special. \\
                    & 8 \replace{(R)} & Friend Stan and Millie laughed together and prepared to spend the morning together. Mamaing Grandma's possibilitant, twice would measure how much she lovedk. Everyone started to get ready when they started arguing until their mum upset. \\
        \hline
        GPT2 (9M) & 5 \accumulate{(A)} & Jack was so happy that he took care of the honey. He thought, "I care about the beautiful garden, because it is nice and clean." He started to feed the flower every day. The flower grew bigger and taller, and Jack became very happy. \\
                    & 5 \replace{(R)} & After playing, Lily got tired and quickly ran back to playing with her dolls. She opened her eyes and played with her dolls all day long. Her grandma was so happy that she screamed as she watched her look back at her original clothes and laughed. 
                    \\
                    & 10 \replace{(R)} & When she finished eating it, she tasted it all up. She said goodbye to her mom and said goodbye. Mommy smiled, feeling very proud of her. It was other. She knew that sharing is always easy to share her meal with her mom. 
                    \\
        \hline
    \end{tabularx}
    \caption{\textbf{Data Accumulation Avoids Model Collapse in Language Modeling.} Both 125M-parameter Llama2 as well as 9M GPT-2 models \replace{show decreasing quality when replacing data (R)}, but \accumulate{maintain high-quality text generations when accumulating data (A)}.}
    \label{tab:sample_generations}
\end{table}

\paragraph{Results}
We found that for all architectures, parameter counts, and sampling temperatures, as the number of model-fitting iterations increased, \replace{replacing data led to an increase in test cross entropy} (Fig. \ref{fig:language_modeling_results} top). We also found that for all architectures, parameter counts, and sampling temperatures, as the number of model-fitting iterations increased, \accumulate{accumulating data led to equal-or-lower test cross entropy} (Fig. \ref{fig:language_modeling_results} bottom). Lower temperature (0.3) led to a faster increase in test error than higher temperature (1.0) (Appendix Fig. \ref{fig:language_modeling_temp}), but the trend was consistent for both temperatures. Table~\ref{tab:sample_generations} shows samples of generated texts for GPT2 (9M) and Llama2 (125M) models at model-fitting iterations 3-5 when both accumulating and replacing data, as well as iterations 8-10 (replacing only).

\paragraph{Ablations} 
We ablate for several additional potential confounds beyond generation temperature. First, when accumulating data, subsequent model iterations are trained on larger datasets than when replacing data. To control for this, we also perform experiments in which data is replaced, but the size of the (fully synthetic) dataset is grown to match the training set size in the accumulation regime. We find that model performance still degrades (albeit at a lower rate). This is shown in Appendix~\ref{app:sec:lm_ablations}, Table~\ref{tab:lm_eval_loss}, right-most column.
Second, a possible concern could be that degrading performance when replacing data could be due to low model performance in iteration 1 (and thus the quality of the first synthetic dataset). We control for this by varying the amount of training performed in iteration 1 only and find that this has no significant impact.
Lastly, we find that our results are also consistent across varying dataset sizes and training epochs. These ablations are discussed in Appendix ~\ref{app:sec:linear_regression_numerics}.

\subsection{Diffusion Models on Molecular Conformation Data}
\label{sec:real:subsec:molecular_diffusion}

\begin{figure}
    \centering
    \includegraphics[width=0.9\textwidth]{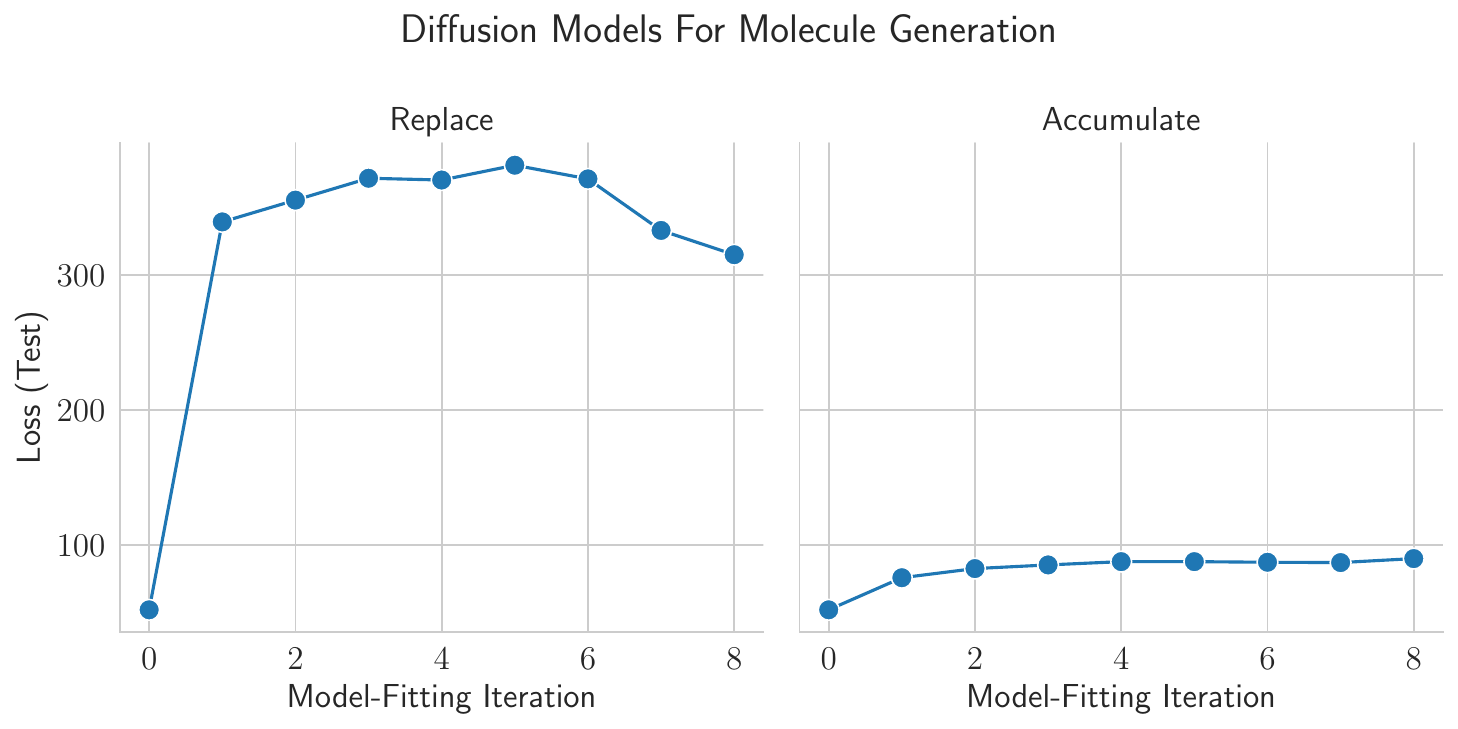}
    \caption{\textbf{Data Accumulation Avoids Model Collapse in Geometric Diffusion Modeling.} GeoDiff, a diffusion-based molecular conformation generation model, is trained on a subset of Drugs data containing molecular structures found in drugs. \replace{Test loss degrades when replacing data (left)} \accumulate{but not when accumulating data (right)}.}
    \label{fig:diff_results}
\end{figure}

\textbf{Experiments} We next train sequences of diffusion models on molecular conformation data. Specifically, we train GeoDiff \citep{xu2022geodiff}, a geometric diffusion model for molecular conformation generation, on the GEOM-Drugs \citep{axelrod2022geom} dataset. We down-sample the training split of GEOM-Drugs to $40,000$ molecular conformations, which we use as our initial training set, and perform $50$ diffusion steps for each prediction. For the loss, we use the standard loss used by GeoDiff: a weighted variational lower bound to the conditional likelihood; for more details, see \citet{xu2022geodiff}.

\textbf{Results} Over $8$ model-fitting iterations, we find \replace{test loss increases when replacing data}, matching our language model experiments, and \accumulate{test loss remains relatively constant when accumulating data} (Fig. \ref{fig:diff_results}). Unlike with language models, we found that when replacing data, performance worsens significantly in the first model-fitting iteration trained on synthetic data and does not degrade further substantially in subsequent iterations.

\subsection{Variational Autoencoders on Image Data}
\label{sec:real:subsec:image_vaes}

\paragraph{Experiments} We lastly train sequences of variational autoencoders (VAEs) \citep{kingma2013auto, rezende2014stochastic} on CelebA \citep{liu2015faceattributes}, a dataset of 200k images of human faces split between train and test sets, chosen as a balance between being a realistic dataset with many samples, color images and resolution, and computational feasibility of training multiple iterations of models on accumulating data. The loss is the standard VAE loss: reconstruction error plus the KL divergence between the encoder's output Gaussian and the isotropic Gaussian prior. See Appendix~\ref{app:sec:vae_ablations} for more experimental details.

\begin{figure}
    \centering
    \includegraphics[width=0.9\textwidth]{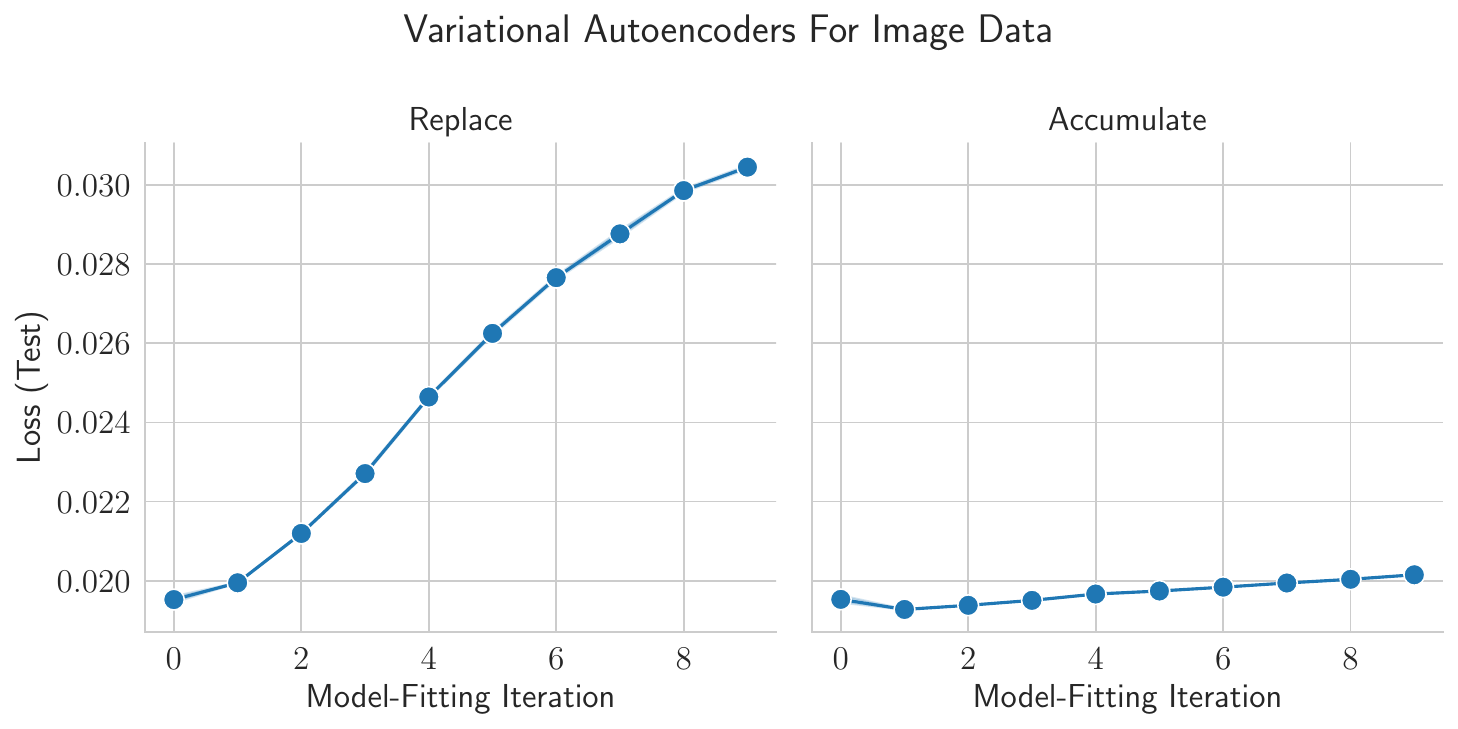}    
    \caption{\textbf{Data Accumulation Avoids Model Collapse in Variational Autoencoders for Image Generation.} Sequences of variational autoencoders (VAEs) are trained on CelebA, a large-scale dataset of human faces. \replace{Test loss degrades when replacing data (left)} \accumulate{but not when accumulating data (right)}.}
    \label{fig:vae_loss}
\end{figure}

\begin{figure}
    \centering
    \includegraphics[width=0.325\textwidth]{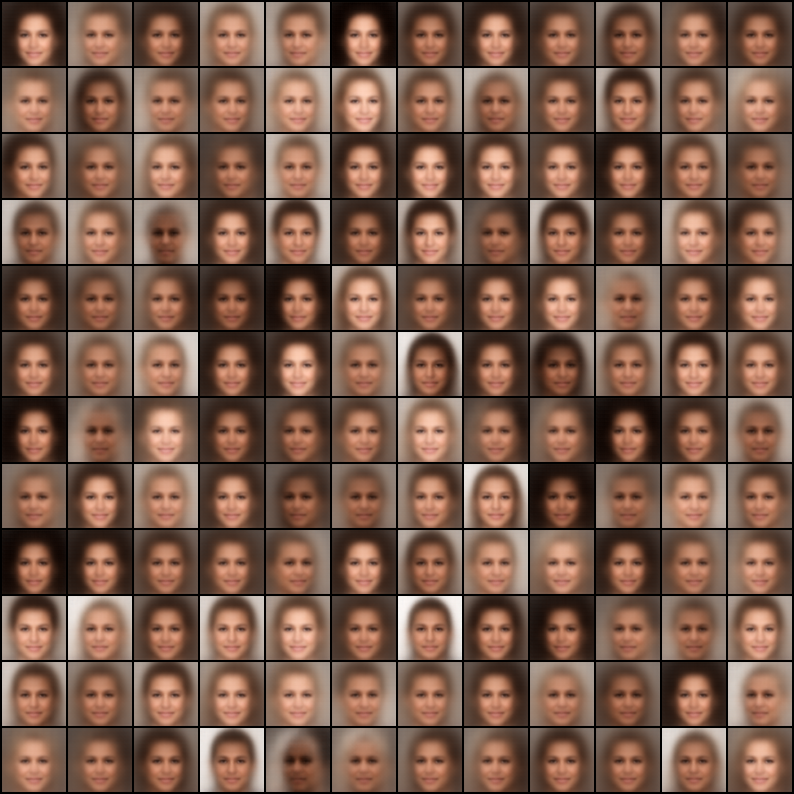}
    \includegraphics[width=0.325\textwidth]{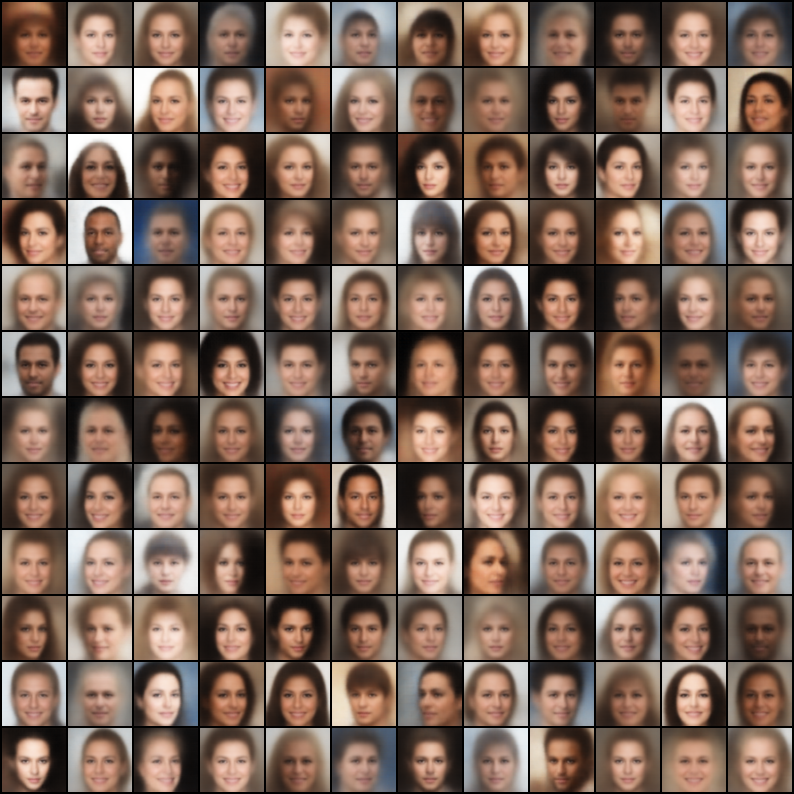}
    \includegraphics[width=0.325\textwidth]{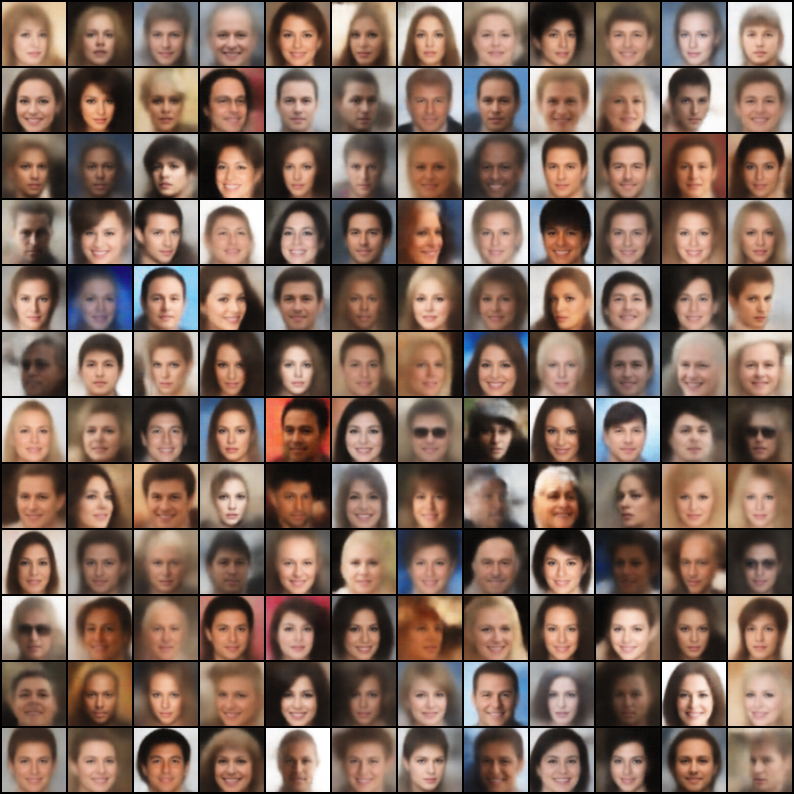}
    \caption{\textbf{Sampled Images from } Left: \replace{Replacing} data with data generated by the previous iteration's newly trained VAE yields lower quality and eventually leads to complete mode collapse. Middle: \accumulate{Accumulating} data with data generated by the previous iteration's newly trained VAE preserves the quality and diversity of generated data across iterations. Right: Baseline samples after 100 training epochs on the dataset.}
    \label{fig:vae_samples}
\end{figure}

\paragraph{Results} We find that \replace{replacing data at each iteration again exhibits model collapse}: the test error rises swiftly with each additional iteration, and each iteration yields lower quality and less diverse generated faces until all model generations represent a single mode as shown in the left panel of Figure \ref{fig:vae_samples}. In contrast, \accumulate{accumulating data at each iteration significantly slows model collapse}: the test error increases significantly slower with each additional iteration. While the diversity of generations does go down as compared in the middle and right panel of Fig. \ref{fig:vae_samples}, it still represents major axes of variation in the dataset, such as gender, but no longer seems to generate other details, along more minor axis of the data manifold, such as glasses and accessories. We discuss further analysis of VAE reconstructions in Appendix~\ref{app:sec:vae_ablations}.

Interestingly, unlike language modeling, the test error of accumulating data does increase with the number of iterations (albeit much more slowly than with replacing data). We also note that \cite{martinez2023combining} found slightly contradictory evidence, specifically that a different architecture on a much smaller dataset exhibits fast performance deterioration even with accumulating data. Understanding under what conditions and why these discrepancies exist is an interesting direction we leave for future research.

\section{Accumulating Data Avoids Model Collapse in Linear Models}
\label{sec:theoretical_setup}

To gain mathematical understanding and intuition, we employ an analytical framework 
introduced in prior work \citep{mobahi2020self, dohmatob2024model} to understand 
the difference between data accumulation and data replacement. We will show that it predicts the 
same types of test error behaviors for these two data-use strategies
that were measured empirically.
The framework considers a sequence of linear models that are fit to the synthetic data sampled from the linear generative model model based on the previously fit linear models.
Within this framework, \citet{dohmatob2024model} showed that \replace{if data are replaced} across model-fitting iterations, \replace{then the test squared error increases linearly}\footnote{To echo an earlier footnote, an approach `halfway' between the `replace' and `accumulate' approaches would replace the previous dataset with a pure synthetic dataset of size $iT$ at the $i$-th iteration. Analyzing this goes mostly in parallel, except the $1/i^2$ mentioned in running text now becomes $1/i$ for the `halfway' approach. Consequently, the MSE scaling becomes $MSE \asymp O(\log(n))$; the `halfway' approach with pure synthetic data but more of it, again has test error growing unboundedly with iterations. Thanks to Elvis Dohmatob, Yunzhen Feng and Julia Kempe for communicating this observation. See Appendix \ref{app:sec:linear_reg_replace_multiple} for an extended discussion.} with the number of iterations $n$. Here, we extend \citet{dohmatob2024model}'s argument to show that \accumulate{if data instead accumulate across model-fitting iterations, then} the test squared error is upper bounded by a relatively small constant, meaning \accumulate{model collapse is avoided}\footnote{In this theoretical section, we identify the term \textit{model collapse} with the situation where test error diverges to infinity (at any rate) as iterations progress. Other authors may employ similar terminology while identifying it with different properties of test error. For example, \cite{alemohammad2023self} use the term MAD to refer to the situation where the distance between the distribution of the original data and that of the subsequent generative models grow farther apart, without necessarily diverging.}.

The content in this section relies heavily on the framework and pioneering contributions of \citet{dohmatob2024model}. Our contribution is to study a  {\it different way to use synthetic data in training}, namely accumulate, which seems to better align with certain real-world considerations. We show that our empirical results could have been anticipated on theoretical grounds, by applying the same analysis framework as in \citet{dohmatob2024model}, but instead to this specific training dataset pattern. We use the same framework to analyze some other ways that synthetic data might have be used, such as replace,  again the theory aligns with many empirical results.

\subsection{Notation and Preliminaries}

\paragraph{Original Data Distribution.} 

We adapt notations from \cite{dohmatob2024model}. Define the distribution $P_{\Sigma,w,\sigma^2}$ on $\mathbb{R}^d \times \mathbb{R}$ given by $(x, y) \sim P_{\Sigma,w,\sigma^2} \quad \text{iff} \quad$:
\begin{align*}
\text{(Input)} \quad & x \sim \mathcal{N}(0, \Sigma), \\
\text{(Noise)} \quad & \epsilon \sim \mathcal{N}(0, \sigma^2), \text{ independent of } x, \\
\text{(Label)} \quad & y = x \cdot w^* + \epsilon.
\end{align*}
The positive integer $d$ is the input-dimension, the matrix $\Sigma \in \mathbb{R}^{d \times d}$ is the true covariance structure of the input $x$, the vector $w^*$ is the true linear relationship used to generate the original data and the scalar $\sigma$ is the level of label noise. We start at iteration $n=1$ with $T$ initial independent data points $(x_i, y_i)$ each following $P_{\Sigma, w^*, \sigma^2}$, that is, $y_i = x_i\cdot w^* + \epsilon_i$ for each $i=1,2,\cdots, T$. We form the design matrix $X\in\R^{T\times d}$ with $x_1^\top,\cdots, x_T^\top$ as rows. We also form the vectors $Y$ and $E$ with $i$-th coordinate $y_i$ and $\epsilon_i$ respectively. In whatever follows, we will assume that $X$ has full column rank, i.e., $T\geq d$, $X^\top X$ is invertible and the model is underparameterized.

\paragraph{Synthetic Data Generation Process.}

We generate synthetic data from the following sequence of distributions

\begin{align*}
P_{\Sigma,w^*,\sigma^2} \to P_{\Sigma,\hat{w}_1,\sigma^2} \to \ldots \to P_{\Sigma,\hat{w}_n,\sigma^2},
\end{align*}
where $n \in \mathbb{N}$ is the number of iterations. The scheme is outlined as follows.
\begin{itemize}
\item For $n=1$:
\begin{itemize}
\item Accumulating Covariates/Features: $\tilde{X}_1 \defeq X$
\item Accumulating Targets: $\tilde{Y}_1 \defeq \hat{Y}_1 \defeq Xw^* + E_1$, where $E_1 \defeq E \sim \mathcal{N}(0, \sigma^2 I_T)$
\item Fit linear model: $\hat{w}_1 = \tilde{X}_1^{\dagger} \tilde{Y}_1$
\item Sample synthetic data for the next iteration: $\hat{Y}_2 \defeq X\hat{w}_1 + E_2$, where $E_2 \sim \mathcal{N}(0, \sigma^2 I_T)$
\end{itemize}
\item For $n \geq 2$:
\begin{itemize}
\item Accumulating Covariates/Features: $\tilde{X}_n^\top = [\tilde{X}_{n-1}^\top; X^\top] \in \mathbb{R}^{d \times nT}$
\item Accumulating Targets: $\tilde{Y}_n^\top = [\tilde{Y}_{n-1}^\top; \hat{Y}_n^\top] \in \mathbb{R}^{1 \times nT}$
\item Fit linear model: $\hat{w}_n \defeq \tilde{X}_n^{\dagger} \tilde{Y}_n$
\item Sample synthetic data for the next iteration: $\hat{Y}_{n+1} \defeq X \hat{w}_n + E_{n+1}$, where $E_{n+1} \sim \mathcal{N}(0, \sigma^2 I_T)$
\end{itemize}
\end{itemize}
Here, for a matrix $A$ with full column rank, $A^\dagger=(A^\top A)^{-1}A^\top$ is the Moore-Penrose pseudo-inverse of $A$. The noise terms $E_1, E_2, \ldots, E_n$ are independent of each other and of the covariates/features. Since $X$ has full column rank, so does $\tilde{X}_n$ for every $n\geq 1$.

\paragraph{Test Error.}

We are interested in the dynamics of the test error $E_{\text{test}}(\hat{w}_n)$ of this sequence of linear model $\hat{w}_1, \hat{w}_2, ...$. Note that evaluation of the model is done on the true distribution $P_{\Sigma,w^*,\sigma^2}$, even though the model is trained on the accumulated synthetic data. For any linear estimator $\hat{w}$ computed from the training data, we measure test error in the standard way:

\begin{equation}
    E_{test}(w) \defeq \mathbb{E}\left[(x_{test}^T w - y_{test})^2 \right] - \sigma^2 = \E[\|w - w^*\|_{\Sigma}^2]
\end{equation}
where the expectation is taken over the training data and $(x_{test}, y_{test})\sim P_{\Sigma,w^*,\sigma^2}$ independent of the training data.

\paragraph{A Note on Extensions to Ridge Regression and Kernel Methods.}

To reiterate a comment made previously by \citet{dohmatob2024model}, although we present our results in the context of ordinary linear regression in $\mathbb{R}^d$, our analysis can be readily extended to ridge regression and the kernel setting \citep{caponnetto2007optimal, simon2021neural, cui2021generalization, wei2022more}. We focus here on a simple useful model for studying model collapse.

\begin{figure}
    \centering
    \includegraphics[width=\textwidth]{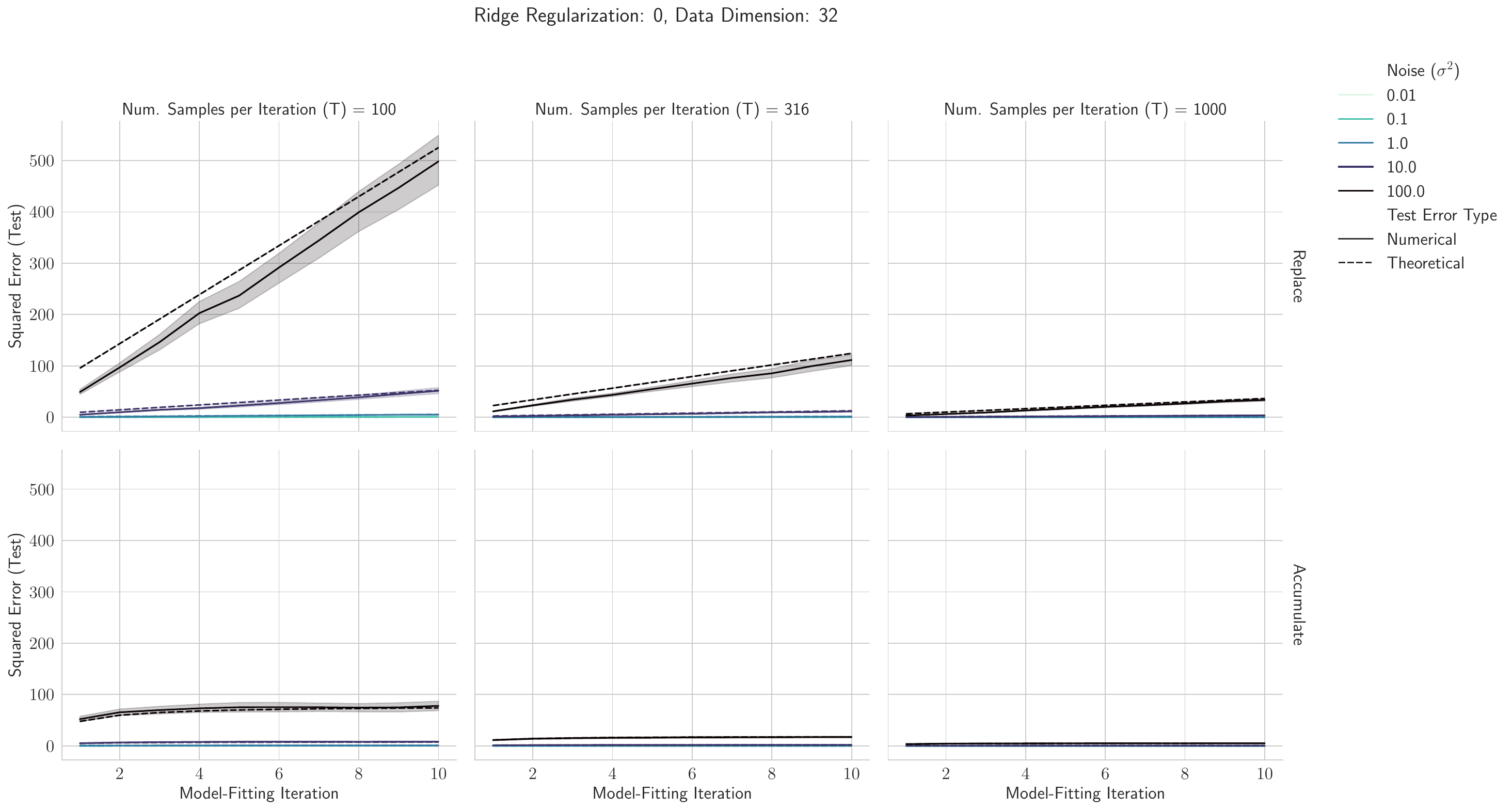}
    \caption{\textbf{Accumulating Data Avoids Model Collapse in Linear Regression.} We consider sequences of linear models recurrently fit to generated targets by previous iterations of models. Top: If each linear model is fit to the generated targets of \textit{only} the preceding linear model, i.e., data are replaced, then the test error grows linearly with the number of iterations $n$. Bottom: If each linear model is instead fit to the generate targets of \textit{all} the preceding linear models, i.e., data accumulate, then the test error has a finite upper bound independent of the number of iterations. This suggests that data accumulation might be a robust solution for mitigating model collapse. For log test error and higher iterations, see Appendix Fig. \ref{app:fig:accumulating_vs_nonaccumulating_isotropic_features_log_linear}.}
    \label{fig:accumulating_vs_nonaccumulating_isotropic_features}
\end{figure}

\subsection{Precise Test Error Characterization Under Accumulating Data}
\label{sec:exact_test_error_characterization}

Our goal is to establish an analytic formula for the test error of the $n$th model in the data accumulation setting. We begin by characterizing the relationship between the fitted linear parameters $\hat{w}_n$ and the true parameters $w^*$. We remind the reader that we assume that $X$ has full column rank, i.e., $X^\top X$ is invertible. Proofs are deferred to App. \ref{app:sec:proofs}.

\begin{theorem}
In the data accumulation setting, $\forall n \geq 1$, the fitted linear parameters $\hat{w}_n$ can be expressed as:
\begin{equation}
\hat{w}_n = w^* + (X^\top X)^{-1} X^\top \left(\sum_{i=1}^n \frac{E_i}{i}\right)
\end{equation}
where, recall, $w^*$ is the true parameter, $X$ is the original design matrix, and $E_i$ is the extra noise added at the $i$'th iteration.
\label{theorem:ridgeless_fitted_linear_parameters}
\end{theorem}

\begin{theorem}
For an $n$-fold synthetic data generation process with $T \geq d + 2$ samples per iteration and isotropic features ($\Sigma \defeq I_d$), the test error for the ridgeless linear predictor $\hat{w}_n$ learned on the accumulated data up to iteration $n$ is given by:
\begin{equation}
E_{\text{test}}^{\text{Accum}}(\hat{w}_n) = \frac{\sigma^2 d}{T-d-1} \left(\sum_{i=1}^n \frac{1}{i^2} \right) \leq \frac{\sigma^2 d}{T-d-1} \times \frac{\pi^2}{6}
\end{equation}
where, recall, $\sigma^2$ is the noise variance of the fake data generation process, $d$ is the input dimension, and $T$ is the number of samples (i.e., data points) added per iteration.
\label{thm:warmup_ridgeless_isotropic}
\end{theorem}

How does test error with accumulating data compare against test error with replacing data? Under otherwise identical assumptions, \citet{dohmatob2024model} proved \replace{in the data-replacing setting} that the test error is given by\footnote{For notational simplicity, we assume that \citet{dohmatob2024model}'s $T_0 \defeq T$ and $\sigma_0 \defeq \sigma$.}:
\begin{equation}
E_{\text{test}}^{\text{Replace}}(\hat{w}_n) = \frac{\sigma^2 d}{T - d -1} \times \color{red}n \color{black}
\label{dohmatob2024thm4p1}
\end{equation}

When data are replaced, the test error grows linearly with the number of iterations $n$ (Fig \ref{fig:accumulating_vs_nonaccumulating_isotropic_features} top), with the rate of growth determined by a noise-to-signal ratio: the amount of noise per dimension $\sigma^2$ times the number of dimensions $d$, adjusted by the (per-iteration) sample size $T$. \accumulate{In contrast, when data accumulate}, Theorem \ref{thm:warmup_ridgeless_isotropic} shows the test error is upper bounded \textit{regardless of the number of iterations $n$}:
\begin{equation*}
E_{\text{test}}^{\text{Accum}}(\hat{w}_n) \leq \frac{\sigma^2 d}{T-d-1} \times \color{red}\frac{\pi^2}{6}\color{black}
\end{equation*}

This striking difference can be intuitively explained by the differences in the way data are handled across iterations. In the \replace{data replacement setting}, because previous data were discarded, the model is more strongly affected by the new noise that each iteration of generated data introduces, and adds that to the effects experienced in earlier iterations. But in the \accumulate{data accumulation setting}, because iteration $i$ contributes fraction $1/i$ to the training dataset, the additional noise from the $i$th iteration of synthetic data has its effect on the model MSE shrunken proportional to $1/i^2$ (due to squared error). The summability of $1/i^2$ prevents the test error from growing indefinitely. This suggests that accumulating generated data with real data can indeed avoid model collapse.

\subsection{Numerical Confirmation of Analytical Results}

To confirm the analytical results, we numerically simulate the setup. The numerics almost perfectly matched the analytics (Fig. \ref{fig:accumulating_vs_nonaccumulating_isotropic_features}): \replace{when data are replaced, the test error grows with the number of iterations $n$}, with the prefactor set by the noise-to-signal ratio $\sigma^2 d / (T - d -1)$, but \accumulate{when data accumulate, the test error rapidly plateaus} with the prefactor similarly set. For log test error and higher model-fitting iterations, see Appendix Fig. \ref{app:fig:accumulating_vs_nonaccumulating_isotropic_features_log_linear}.

\section{Discussion}
\label{sec:discussion}

This work explored the phenomenon of model collapse, an important concern as AI-generated content permeates the internet and finds its way into future training datasets. Prior work has shown that training on model outputs can lead to degraded performance \citep{martinez2023combining,martinez2023towards,shumailov2023curse,alemohammad2023self,hataya2023will,bertrand2023stability,briesch2023large,dohmatob2024model,dohmatob2024tale}, implying that future model training faces a difficult challenge of ensuring strict training dataset hygiene.
\textbf{For a significantly more thorough discussion of related work, please see Appendix \ref{app:sec:prior_work}.}

Our findings extend these prior works to show that if data \accumulate{accumulates} and models train on a mixture of ``real'' and synthetic data, \accumulate{model collapse no longer occurs}. We show this both experimentally on causal transformers for language modeling, diffusion models for molecule generation, and variational auto-encoders on image data as well as theoretically for linear regression. \textbf{Together, these results strongly suggest that the ``curse of recursion" may not be as dire as had been portrayed \--- provided we accumulate synthetic data alongside real data, rather than replacing real data by synthetic data only.}

Looking to the future, many questions worth investigating remain. For instance, in future work we would like to explore different data generation and accumulation regimes, such as (1) additional ``real'' data being introduced in each model-fitting iteration and (2) different schedules of how much synthetic data is generated at each iteration and (3) human-filtering of generated data, e.g., as done in RLHF. Additionally, we note that in all our experiments, the synthetic dataset is generated by \textit{sampling} from the previous model, i.e., with some stochasticity; in future work, we would like to explore also what happens if data is generated deterministically, e.g. with temperature 0 in a typical language model.

Lastly, it is worth noting that ``model collapse" -- as a term of art -- has been used in various ways by various researchers; so care is required in comparing claims across articles. In reviewing the literature, we identified at least four related phenomena: (0) unbounded test error blowup (as here); (1) modal collapse --- collapse to one (or a few) modes; (2) collapse to uniformity; and (3) amplification of artifacts introduced by models fit to previous synthetic data. Future work should map out what factors cause which to occur and what preventative strategies are effective at addressing each.

\section{Acknowledgements}

The content of this paper does not necessarily reflect the position or the policy of any of the funding agencies/entities. No endorsement should be inferred. M.G. acknowledges support through a grant from the Cooperative AI Foundation. R.S. acknowledges support from Stanford Data Science and from OpenAI's Superalignment Fast Grant Research Fellowship.  A.G. acknowledges support from the NSF CAREER grant DMR-2045181, the Sloan Foundation, and by the Laboratory
for Physical Sciences through the Condensed Matter Theory Center.
D.R. acknowledges support from the National Science Foundation under Cooperative Agreement PHY-2019786 (the NSF \href{http://iaifi.org/}{AI Institute for Artificial Intelligence and Fundamental Interactions}) and appreciates both the sanction and support of Sequoia Capital.
S.K. is partially supported by NSF III 2046795, IIS 1909577, CCF 1934986, NIH 1R01MH116226-01A, NIFA award 2020-67021-32799, the Alfred P. Sloan Foundation, and Google Inc.

\clearpage

\bibliography{colm2024_conference}
\bibliographystyle{colm2024_conference}

\clearpage
\appendix
\section{Summarization and Discussion of Prior and Related Work}
\label{app:sec:prior_work}

\textbf{Prior Empirical Work} A growing body of recent work has investigated the phenomenon of iteratively training models on data generated by previous models, e.g., \citet{hataya2023will,martinez2023combining,shumailov2023curse, alemohammad2023self,martinez2023towards,bertrand2023stability,briesch2023large, dohmatob2024model, dohmatob2024tale} and (in a different context) \citet{taori2023data}.
\citet{hataya2023will} and \citet{martinez2023towards} conducted experiments replacing real training data with generated data at each iteration, assuming that the dataset size remains fixed over time. They found that this iterative retraining procedure can lead to model degradation if the proportion of synthetic data becomes too high. Similarly, \citet{shumailov2023curse} ran experiments with Gaussian mixture models, VAEs, and language models in which the total number of samples per iteration was held constant, and the samples always originated with the previous model rather than aggregating over time. Building on this work, \citet{alemohammad2023self} considered three treatments of data: fully replacing real data with synthetic data, augmenting a fixed real dataset with additional synthetic data, and mixing new real data with synthetic data at each iteration. In almost all of their experiments, they drew a fixed size dataset from the most recent model at each iteration, without accumulating data. \citet{bertrand2023stability} also assumed that dataset size and mixing proportions are constant over time in their theoretical stability analysis and empirical validation. %

\paragraph{Prior Theoretical Work} Over the last few years, there has been significant research effort contributing to our theoretical understanding of model behavior when synthetic data are integrated into training. The most closely related works to ours are \citet{dohmatob2024model} and \citet{dohmatob2024tale}; of course, the inspiration for the linear regression model studied in this paper directly comes from \citet{dohmatob2024model}. \citet{dohmatob2024model} performs an in-depth analysis of high dimensional linear and ridge regression when the training data used per iteration are generated from the previous iteration's fitted model. They are able to conclude that the test error grows linearly with the iteration count in their setup, as well as derive more interesting and more nuanced results using random matrix theory. They also discuss how to mitigate model collapse through optimal regularization both when the training data are noise-free and noisy versions of the previous model's synthetic outputs. A related noise-free setup was studied by \citet{mobahi2020self} in the case of self-distillation. Although \cite{mobahi2020self} considers a more general setup with ridge regression as a special case, they use \textit{noiseless} predictions from the previous model as the training data for the next model, and show that eventually, the predictions shrink to zero. Through this, they highlight that self-distillation induces regularization in the function space, which initially is beneficial for reducing over-fitting, but eventually over-regularization causes underfitting and hence performance decay. \cite{dohmatob2024tale} go beyond the linear model to study model collapse -- they study the tails of LLM outputs vs. real data and provide scaling laws that clearly identify regimes of model degradation when synthetic data misses \textit{tails} present in real data. They identify an interesting \textit{phase transition} in the test error scaling law depending on the size of the real dataset size in comparison to (a functional of) the chopped-off tail, and conclude that enough real data is able to mitigate model collapse. All these works consider the scenario where the amount of training data available per iteration is fixed (and does not grow with the iteration count), and it is certainly possible that with larger amount of synthetic data (from prediction by the previous model), several of these scalings would improve significantly. For example, in Equation (12) of \cite{dohmatob2024tale}, one obtains the linear scaling (with iteration count) of test error simply because the amount of synthetic data generated per iteration is the same. If one generated synthetic data with size proportional to the iteration count, then at iteration $n$, the scaling would, instead of $n$, be like $n^{1-c}/(1-c)$ for $c<1$. When one does not increase the dataset size, \cite{dohmatob2024tale} points out that increasing the proportion of real data would help one to avoid model collapse altogether. However, even if one \textit{did} increase the amount of synthetic data with iteration count, Theorem 3.2 coupled with Corollary 3.3 in \cite{dohmatob2024tale} would tell us that the \textit{amount} of real data was all that mattered -- if the amount of real data is large, we overcome model collapse. If one only had synthetic data (and no real data), no matter how large, it would be impossible to regain the original real-data scaling laws.
The scenario we study is highly inspired by these pioneering works, but still, in our view, different. We consider the case when we keep \textit{augmenting} synthetic data (generated by the previous model trained on all the previous data so far) as iterations progress, much akin to how -- in our view -- the internet evolves. We observe that we can avoid model collapse in this setting. The analysis of previous models in our case is more involved, since the data used for training at iteration $n$ is not homogeneous -- different models from the past impart different statistical aspects to different parts of the training data. We also note a related augmentation model studied by \cite{jain2024scaling} -- they perform risk minimization augmenting real data with synthetic data available from a potentially different independent source. One of their messages is that augmentation of (even) pure noise can be looked upon as adding a ridge penalty and hence, in certain cases, can improve test error. Their setup, however, is different from ours, since the synthetic data in their setup is not obtained by a learning algorithm employed on the real data, and the process is not iterative. However, morally, each iteration of ours involves risk minimization on data statistically composed of an equal mixture of data generated from the previous models, and hence each iteration of ours can be mapped to the general framework developed in \cite{jain2024scaling}, although the dependencies among the various models trained in our setup introduce theoretical complications that do not seem to be too easily addressed by the theory developed in \cite{jain2024scaling}. Shortly after v1 of our manuscript was uploaded to ArXiv, two other manuscripts appeared, dealing with the theoretical aspects in a setting similar to ours. Theorem 1 of \cite{marchi2024heat} obtains the same square summability scaling of the variance as us. \cite{seddik2024bad} studies collapse in language models in both purely synthetic and partly synthetic regimes and obtains deviation bounds as model iterations progress.

\begin{figure}[b!]
    \centering
    \includegraphics[trim={0 10cm 0 0},clip,width=\textwidth]{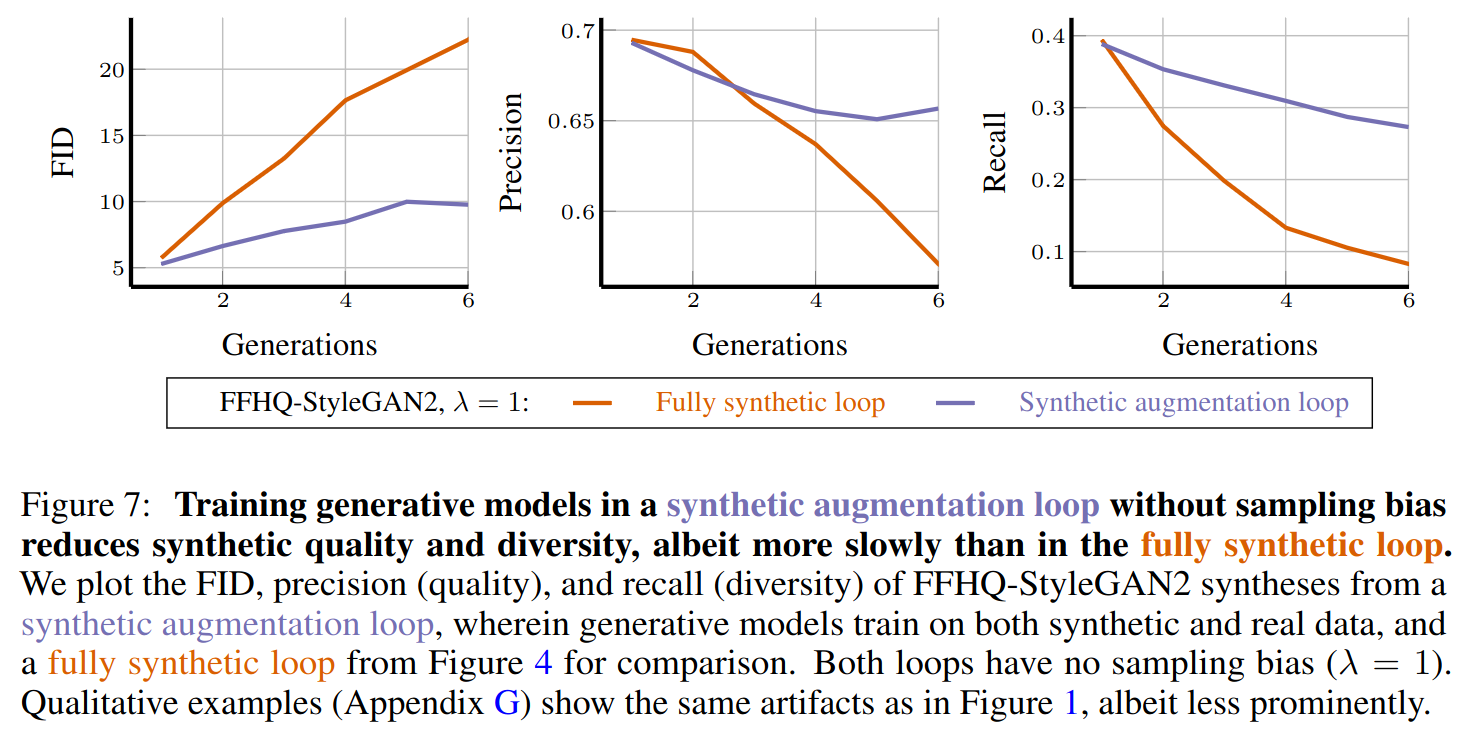}
    \caption{\textbf{Clarification of Data Accumulation in \citet{alemohammad2023self}}. Figure 7 from \citet{alemohammad2023self} (above) shows that linearly accumulating data (``Synthetic augmentation loop") causes poor behavior to plateau with the number of model-fitting iterations. \citet{alemohammad2023self} write, ``Our experiments [...] support our main conclusion [that] fixed real training data only delays the inevitable degradation of the quality or diversity of the generative models over generations." We believe is that our evidence and their evidence is more consistent with the conclusion that accumulating data \textit{avoids} model collapse and does not merely delay it.}
    \label{app:fig:responding_to_alemohammad2023selfconsuming}
\end{figure}

\paragraph{Considering Accumulating Data} The two papers we found that partially considered accumulating data are \cite{martinez2023combining} and \cite{alemohammad2023self}. \cite{alemohammad2023self} did so in one-half of one experiment: StyleGAN2 trained on FliqrFaces 128$\times$128 (App. Fig. \ref{app:fig:responding_to_alemohammad2023selfconsuming}). The authors concluded that accumulating data does not avoid model collapse, but merely slows it down. However, we believe that a closer examination of their results (App. Fig. \ref{app:fig:responding_to_alemohammad2023selfconsuming}) reveals that accumulating data causes the test error to plateau to a relatively low error with increasing numbers of model-fitting iterations. This result would support our conclusion that accumulating data avoids model collapse and does not merely delay it. The results from \citet{martinez2023combining} are harder to evaluate; 
model collapse only seems to occur when the amount of synthetic data added per model-fitting iteration is 2$\times$ the total amount of accumulated data, and the subsequent work by the authors switched from accumulating data to replacing data \citep{martinez2023towards}. We think understanding what conditions and why these discrepancies exist is an interesting future direction.

\paragraph{Avoiding Model Collapse} Several papers present methods for avoiding or slowing model collapse. \citet{bertrand2023stability} shows in the replacing data setting that model collapse will not occur if the initial generative models approximate the data distribution well enough and the proportion of real data is sufficiently large with respect to the synthetic data.
\citet{dohmatob2024tale} similarly demonstrates that in the replacing data setting, carefully selecting real data to mix with synthetic data can avoid model collapse. Other solutions may also be possible in various models and under various assumptions. To our knowledge, no paper has claimed an ``optimal" strategy to avoid model collapse, and neither has ours.

\clearpage
\section{Proofs of Mathematical Results}
\label{app:sec:proofs}

We point out a lemma useful to prove Theorem \ref{thm:warmup_ridgeless_isotropic}.
\begin{lemma}
Let $T$ and $d$ be positive integers with $T \geq d + 2$, and let $X \in \mathbb{R}^{T \times d}$ be a random matrix with i.i.d. rows from $\mathcal{N}(0, \Sigma)$ with $\Sigma$ positive definite. Then, $X$ has full rank a.s. Moreover, it holds that:
\begin{equation}
\mathbb{E}_X[(X^\top X)^{-1}] = \frac{1}{T - d - 1} \Sigma^{-1}.
\end{equation}
\label{lemma:tr_inv_cov_eq_prefactor_inv_cov}
\end{lemma}

\begin{proof}
    See \citet{dohmatob2024model}.
\end{proof}

Assuming Lemma \ref{lemma:tr_inv_cov_eq_prefactor_inv_cov} and Theorem \ref{theorem:ridgeless_fitted_linear_parameters}, we present the proof of Theorem \ref{thm:warmup_ridgeless_isotropic}.

\begin{proof}[Proof of Theorem \ref{thm:warmup_ridgeless_isotropic}]
From Theorem \ref{theorem:ridgeless_fitted_linear_parameters}, we have:
\begin{equation}
\hat{w}_n = w^* + (X^\top X)^{-1} X^\top \left(\sum_{i=1}^n \frac{E_i}{i}\right)
\end{equation}
where $w^*$ is the true parameter, $X$ is the original data matrix, and $E_i$ are the noise terms at each iteration, with $E_i \sim \mathcal{N}(0, \sigma^2 I_T)$. The test error is given by:
\begin{equation}
E_{\text{test}}(\hat{w}_n) = \mathbb{E}[||\hat{w}_n - w^*||_{\Sigma}^2]
\end{equation}where the expectation is taken over all random quantities involved.

Substituting $\hat{w}_n$ into the test error expression and using the fact that $\Sigma \defeq I_d$, we get:
\begin{align*}
E_{\text{test}}(\hat{w}_n) &= \mathbb{E}\left[\left(\sum_{i=1}^n \frac{E_i}{i}\right)^\top X(X^\top X)^{-2} X^\top \left(\sum_{i=1}^n \frac{E_i}{i}\right)\right] \\
&= \mathbb{E}\left[\sum_{i=1}^n \frac{\sigma^2}{i^2} \text{tr}(X(X^\top X)^{-2} X^\top)\right] \\
&= \sum_{i=1}^n \frac{\sigma^2}{i^2} \mathbb{E}\left[\text{tr}((X^\top X)^{-1})\right]
\end{align*}

Using Lemma \ref{lemma:tr_inv_cov_eq_prefactor_inv_cov}, we have:
\begin{equation}
\mathbb{E}_{X}\left[\text{tr}((X^\top X)^{-1})\right] = \frac{d}{T-d-1}
\end{equation}

Therefore, the test error for ridgeless regression with isotropic features in the data accumulation setting is:
\begin{align*}
E_{\text{test}}(\hat{w}_n) &= \sum_{i=1}^n \frac{\sigma^2}{i^2} \cdot \frac{d}{T-d-1} < \frac{\sigma^2 d}{T-d-1} \left(\frac{\pi^2}{6}\right)
\end{align*}
as $\sum_{i=1}^n i^{-2} < \sum_{i=1}^\infty i^{-2} = \pi^2/6$.
\end{proof}

Finally, we prove Theorem \ref{theorem:ridgeless_fitted_linear_parameters}.

\begin{proof}[Proof of Theorem \ref{theorem:ridgeless_fitted_linear_parameters}]

We prove this theorem by induction.

\textbf{Base case:} For $n=1$, we have:
\begin{align*}
\hat{w}_1 &= \tilde{X}_1^{\dagger} \tilde{Y}_1 = (X^\top X)^{-1} X^\top (Xw^* + E_1) = w^* + (X^\top X)^{-1} X^\top E_1
\end{align*}
which satisfies the lemma.

\textbf{Inductive step:} Assume that for some $n \geq 1$, we have:

\begin{equation*}
\hat{w}_n = w^* + (X^\top X)^{-1} X^\top \left(\sum_{i=1}^n \frac{E_i}{i}\right)
\end{equation*}

Now, consider $\hat{w}_{n+1}$:
\begin{align*}
\hat{w}_{n+1} &= \tilde{X}_{n+1}^{\dagger} \tilde{Y}_{n+1} \\
&= (\tilde{X}_{n+1}^\top \tilde{X}_{n+1})^{-1} \tilde{X}_{n+1}^\top \tilde{Y}_{n+1} \\
&= \frac{1}{n+1}(X^\top X)^{-1} \sum_{i=1}^{n+1} X^\top \hat{Y}_i
\end{align*}

Recalling that $\hat{Y}_i$:
\begin{align*}
\hat{Y}_i = \begin{cases}
X w^* + E_1, & i = 1 \\
X \hat{w}_{i-1} + E_i, & 2 \leq i \leq n+1
\end{cases}
\end{align*}

Substituting this back into the expression for $\hat{w}_{n+1}$:
\begin{align*}
\hat{w}_{n+1} &= \frac{1}{n+1}(X^\top X)^{-1} \left(X^\top (X w^* + E_1) + \sum_{i=2}^{n+1} X^\top (X\hat{w}_{i-1} + E_i)\right) \\
&= \frac{1}{n+1}(X^\top X)^{-1} \left(X^\top Xw^* + X^\top E_1 + \sum_{i=2}^{n+1} (X^\top X\hat{w}_{i-1} + X^\top E_i)\right) \\
&= \frac{1}{n+1}(X^\top X)^{-1} \left(X^\top Xw^* + X^\top E_1 + \sum_{i=1}^{n} (X^\top X\hat{w}_i + X^\top E_{i+1})\right) \\
&= \frac{1}{n+1}(X^\top X)^{-1} \left(X^\top X w^* + \sum_{i=1}^{n} X^\top X\hat{w}_i + \sum_{i=1}^{n+1} X^\top E_i\right)
\end{align*}

Now, using the induction hypothesis:
\begin{align*}
\hat{w}_{n+1} &= \frac{1}{n+1}(X^\top X)^{-1} \left(X^\top Xw^* + \sum_{i=1}^{n} X^\top X\left(w^* + (X^\top X)^{-1} X^\top \sum_{j=1}^i \frac{E_j}{j}\right) + \sum_{i=1}^{n+1} X^\top E_i\right) \\
&= \frac{1}{n+1}(X^\top X)^{-1} \left((n+1)X^\top Xw^* + \sum_{i=1}^{n} X^\top X(X^\top X)^{-1} X^\top \sum_{j=1}^i \frac{E_j}{j} + \sum_{i=1}^{n+1} X^\top E_i\right) \\
&= w^* + \frac{1}{n+1}(X^\top X)^{-1} \left(\sum_{i=1}^{n} X^\top \sum_{j=1}^i \frac{E_j}{j} + \sum_{i=1}^{n+1} X^\top E_i\right) \\
&= w^* + \frac{1}{n+1}(X^\top X)^{-1} X^\top \left(\sum_{i=1}^{n} \sum_{j=1}^i \frac{E_j}{j} + \sum_{i=1}^{n+1} E_i\right)
\end{align*}

Now, we need to simplify the term $\sum_{i=1}^{n} \sum_{j=1}^i \frac{E_j}{j} + \sum_{i=1}^{n+1} E_i$. We can do this by counting the number of times each $E_i$ appears in the double sum: $E_1$ appears $n$ times in the double sum and once in the single sum, so its coefficient is $\frac{n+1}{1}$. $E_2$ appears $n-1$ times in the double sum and once in the single sum, so its coefficient is $\frac{n}{2}$. This continues along till we reach $E_n$, which appears once in the double sum and once in the single sum, so its coefficient is $\frac{2}{n}$. $E_{n+1}$ appears only once in the single sum, so its coefficient is $\frac{1}{n+1}$. Therefore,
\begin{align*}
\sum_{i=1}^{n} \sum_{j=1}^i \frac{E_j}{j} + \sum_{i=1}^{n+1} E_i &= \sum_{i=1}^{n+1} \frac{n+2-i}{i} E_i = (n+1) \sum_{i=1}^{n+1} \frac{E_i}{i}
\end{align*}

Substituting this back into the expression for $\hat{w}_{n+1}$:
\begin{align*}
\hat{w}_{n+1} &= w^* + \frac{1}{n+1}(X^\top X)^{-1} X^\top \left((n+1) \sum_{i=1}^{n+1} \frac{E_i}{i}\right) \\
&= w^* + (X^\top X)^{-1} X^\top \sum_{i=1}^{n+1} \frac{E_i}{i}
\end{align*}

Therefore, by mathematical induction, the lemma holds for all $n \geq 1$.
\end{proof}

\clearpage
\section{Additional Details and Ablations on Language Model Experiments}
\label{app:sec:lm_ablations}
\subsection*{Implementation Details}
Model training was implemented using Huggingface Transformers~\citep{wolf2019huggingface}. Dataset generation was implemented using vllm~\citep{kwon2023efficient}.

\subsection*{Additional Plots}
In addition to Figure~\ref{fig:language_modeling_results_learning_curves} in the main text, Figures~\ref{app:fig:language_modeling_results_learning_curves_epochs_lin}-\ref{app:fig:language_modeling_results_learning_curves_steps_log} show learning curves in larger print, with x-axes showing either epochs or gradient steps, and with axes shown in linear-linear or log-log scale, respectively.

\begin{figure}
    \centering
    \includegraphics[width=0.9\textwidth]{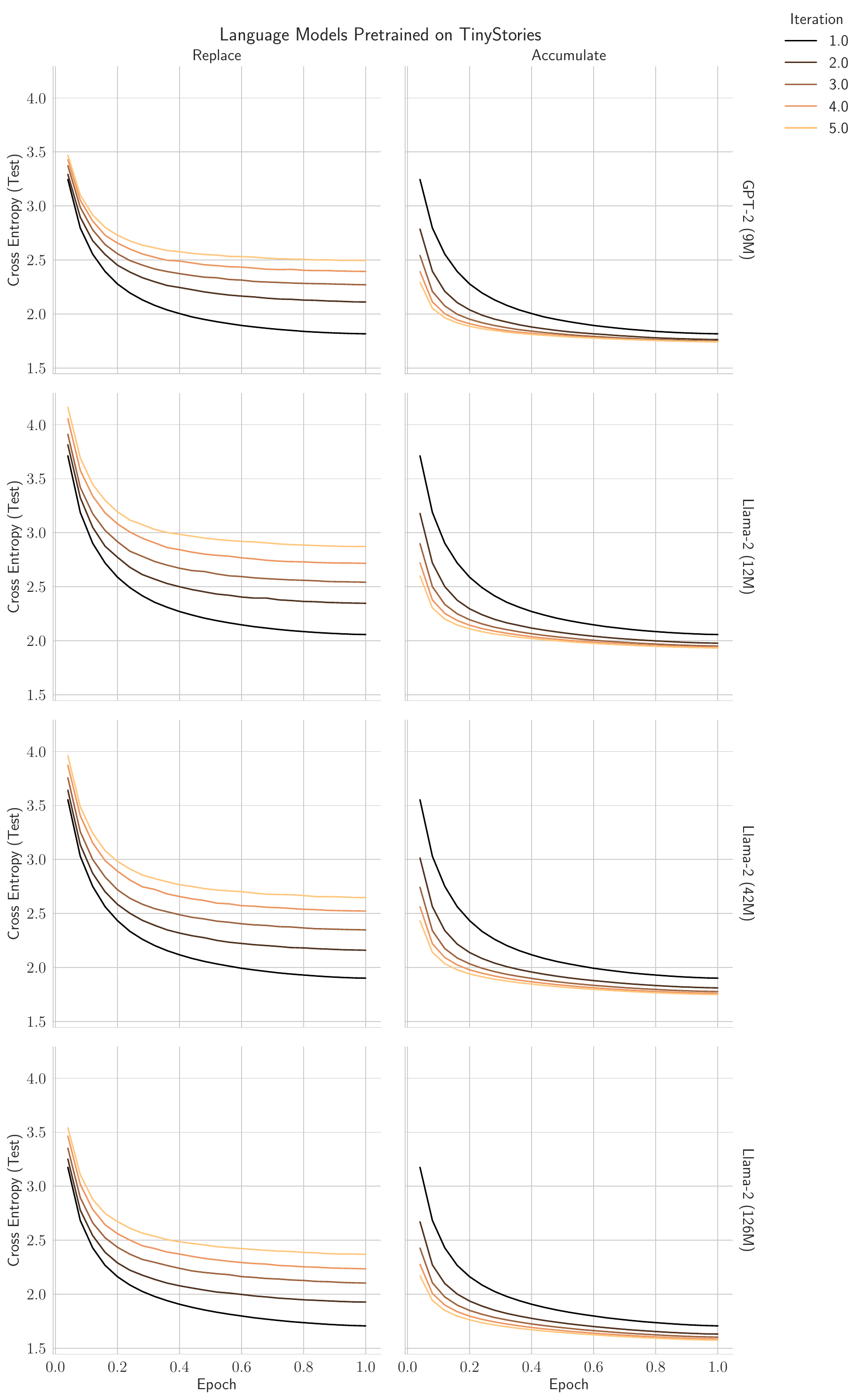}
    \caption{\textbf{Data Accumulation Avoids Model Collapse in Language Modeling.} Learning curves for individual model-fitting iterations when repeatedly \textit{replacing} data (left), and when \textit{accumulating} data (right). Note: Epochs correspond to more gradient steps for accumulate than replace because the number of training data grows for accumulate.
    }
    \label{app:fig:language_modeling_results_learning_curves_epochs_lin}
\end{figure}

\begin{figure}
    \centering
    \includegraphics[width=0.9\textwidth]{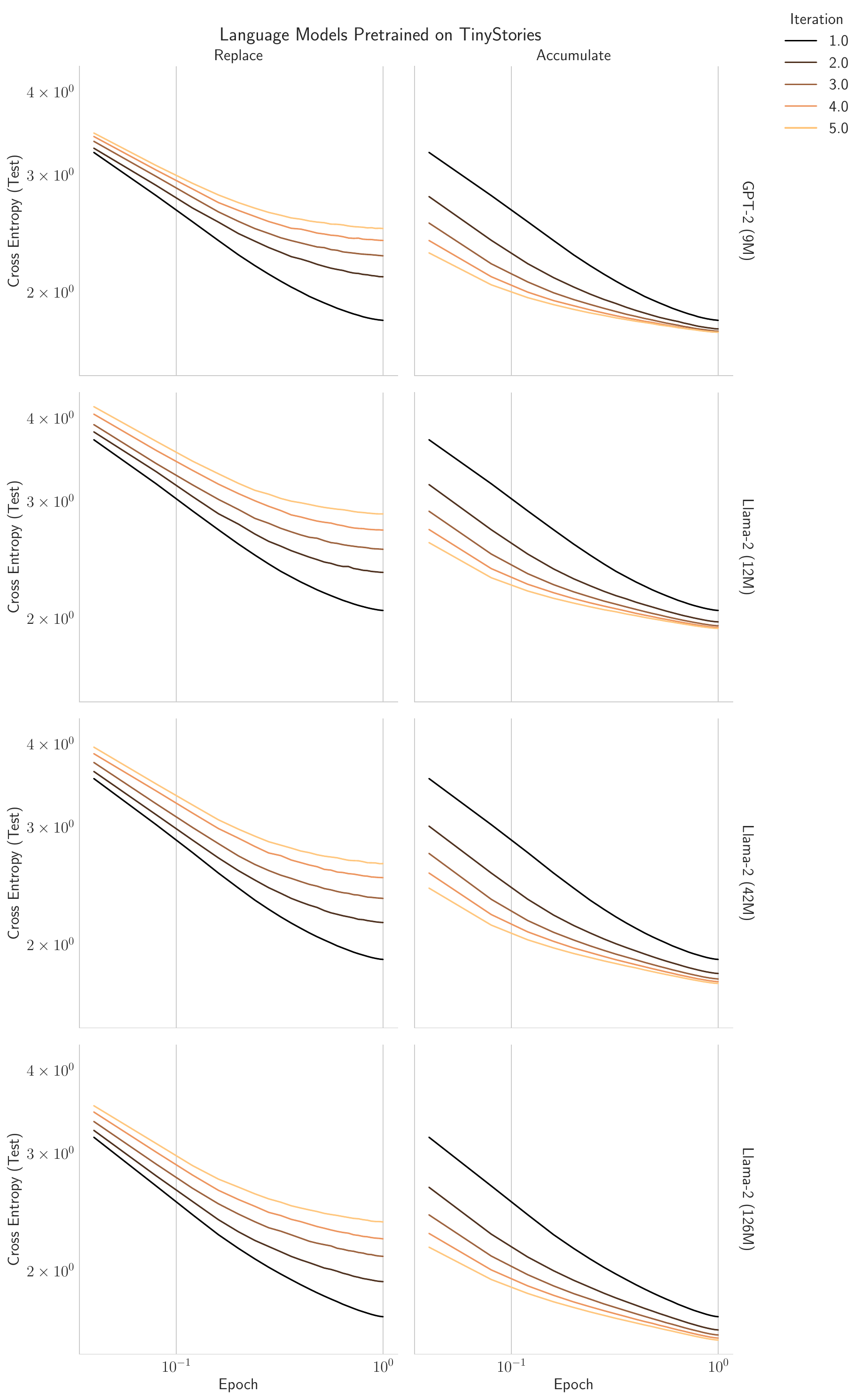}
    \caption{\textbf{Data Accumulation Avoids Model Collapse in Language Modeling.} Learning curves for individual model-fitting iterations when repeatedly \textit{replacing} data (left), and when \textit{accumulating} data (right), in log-log scale. Note: Epochs correspond to more gradient steps for accumulate than replace because the number of training data grows for accumulate.
    }
    \label{app:fig:language_modeling_results_learning_curves_epochs_log}
\end{figure}

\begin{figure}
    \centering
    \includegraphics[width=0.9\textwidth]{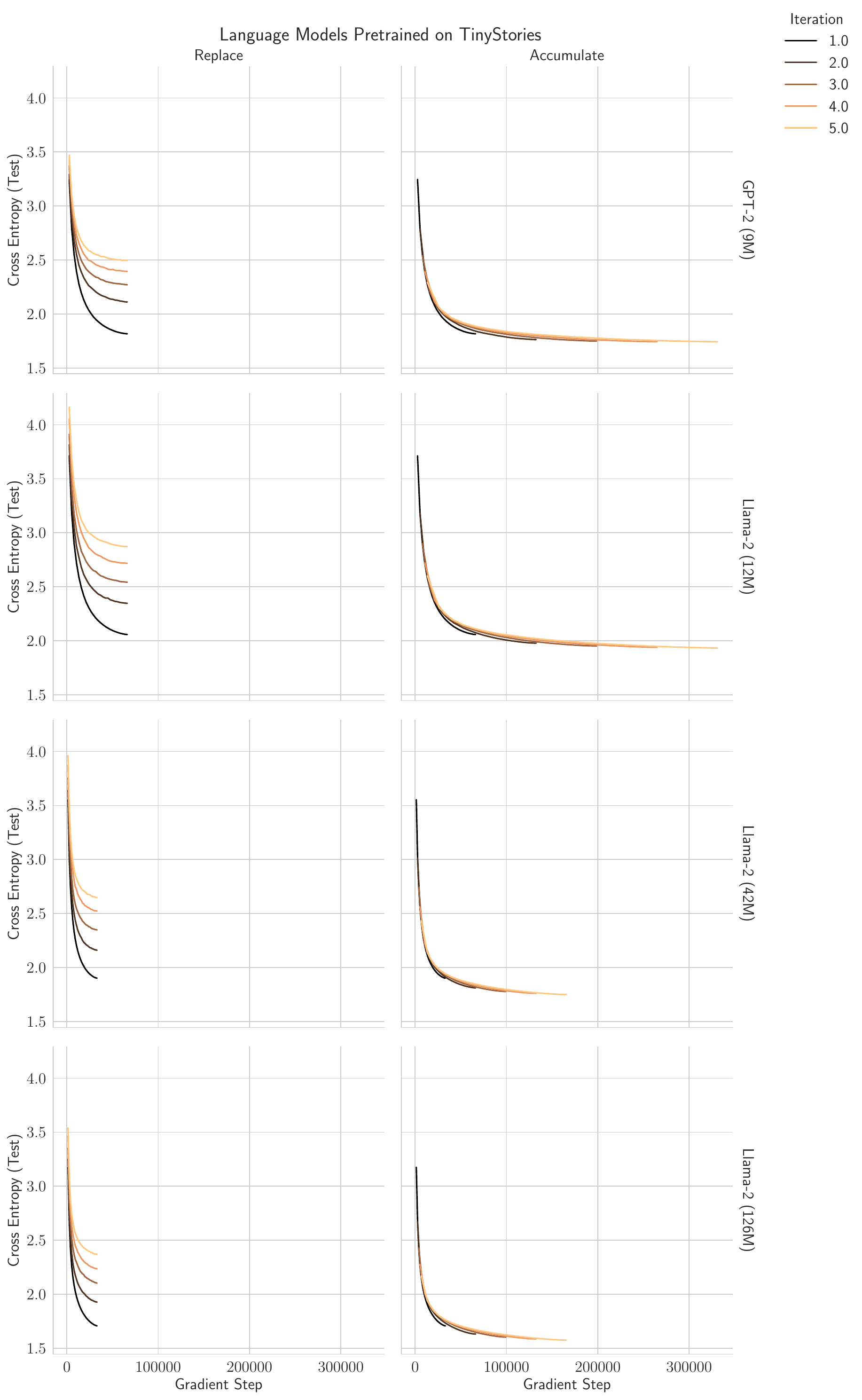}
    \caption{\textbf{Data Accumulation Avoids Model Collapse in Language Modeling.} Learning curves for individual model-fitting iterations when repeatedly \textit{replacing} data (left), and when \textit{accumulating} data (right).
    }
    \label{app:fig:language_modeling_results_learning_curves_steps_lin}
\end{figure}

\begin{figure}
    \centering
    \includegraphics[width=0.9\textwidth]{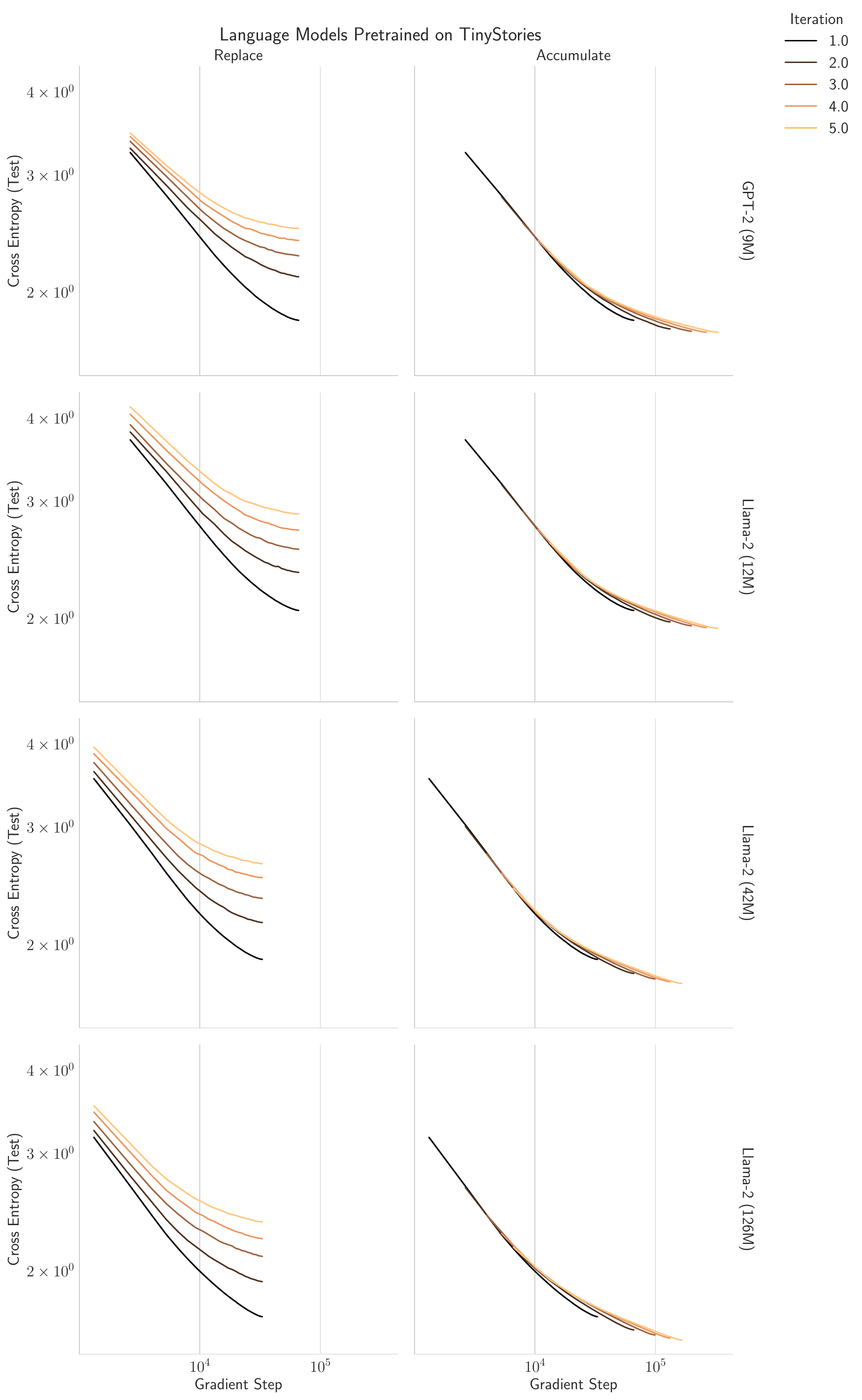}
    \caption{\textbf{Data Accumulation Avoids Model Collapse in Language Modeling.} Learning curves for individual model-fitting iterations when repeatedly \textit{replacing} data (left), and when \textit{accumulating} data (right), in log-log scale.
    }
    \label{app:fig:language_modeling_results_learning_curves_steps_log}
\end{figure}

\subsection*{Ablations}
In addition to the experiments shown in the main paper, we conducted several ablation studies.

\paragraph{Controlling for dataset size.} One possible concern is that when accumulating data, the train dataset size will grow at each model-fitting iteration, meaning subsequent models will be trained on more aggregate data than their counterparts in the replacement regime. To control for this, we run experiments controlling for this. In this ``replace-multiple'' regime, we create a fully synthetic dataset at the end of each model-fitting iteration, but grow the size of this dataset to match that of the accumulated data in the accumulation regime. Table~\ref{tab:lm_eval_loss} rightmost column shows that in this regime, evaluation loss still increases over model-fitting iterations.

\paragraph{Generation temperature.} Most of our language model experiments were run with sampling temperature $1.0$ during generation of new datasets. To ensure that this choice is not critical, we also run one experiment with temperature $0.3$, and see that this shows similar results (with even larger increases in validation loss in the replacement regime than temperature $1.0$), as shown in Table~\ref{tab:lm_eval_loss}, row 2, and Figure~\ref{fig:language_modeling_temp}.

\paragraph{Dataset size and training epochs.} We similarly vary the size of the initial (and subsequent) training datasets and number of training epochs, and see that this has no qualitative effect on the results (Table~\ref{tab:lm_eval_loss}, rows 3 \& 4 show training on 1/5th of the TinyStories dataset for 1 \& 3 epochs, respectively).

\paragraph{Model quality after first model-fitting iteration.} Finally, we control specifically for model (and thus synthetic dataset) quality after the first iteration, to rule out an undue influence of a ``bad'' first synthetic dataset on subsequent training. Figure~\ref{fig:language_modeling_firstiter} shows performance in subsequent iterations for different amounts of training in the first iteration, showing no qualitative differences.

\begin{table}[h!]
    \centering
    \begin{tabular}{c|c|c|c|c|c}
        Model & t=1 & t=4 (acc) & t=4 (repl) & t=10 (repl) & t=4 (*) \\
        \hline 
GPT-2 (9M) & 1.82 & 1.74 (-0.07) & 2.39 (+0.58) & 2.91 (+1.09) & 2.18 (+0.36) \\
GPT-2 (9M) (temp=0.3) & 1.82 & 1.75 (-0.06) & 5.82 (+4.00) & 9.85 (+8.04) & n/a \\
GPT-2 (9M) (small dataset) & 2.56 & 2.28 (-0.28) & 3.21 (+0.65) & 3.72 (+1.16) & 2.91 (+0.35) \\
ibid (+ 3 epochs) & 1.99 & 1.87 (-0.12) & 2.62 (+0.63) & n/a & n/a \\
Llama-2 (12M) & 2.06 & 1.94 (-0.12) & 2.72 (+0.66) & n/a & n/a \\
Llama-2 (42M) & 1.90 & 1.76 (-0.14) & 2.52 (+0.62) & n/a & n/a \\
Llama-2 (126M) & 1.71 & 1.59 (-0.12) & 2.23 (+0.53) & n/a & n/a \\
    \end{tabular}
    \caption{Evaluation cross-entropy loss for different models at model-fitting iterations 1, 4 and 10 for replacement and accumulation regimes. (*) indicates a replacement regime with growing dataset size to ablate for total train set size.}
    \label{tab:lm_eval_loss}
\end{table}

\begin{figure}
    \centering
    \includegraphics[width=0.7\textwidth]{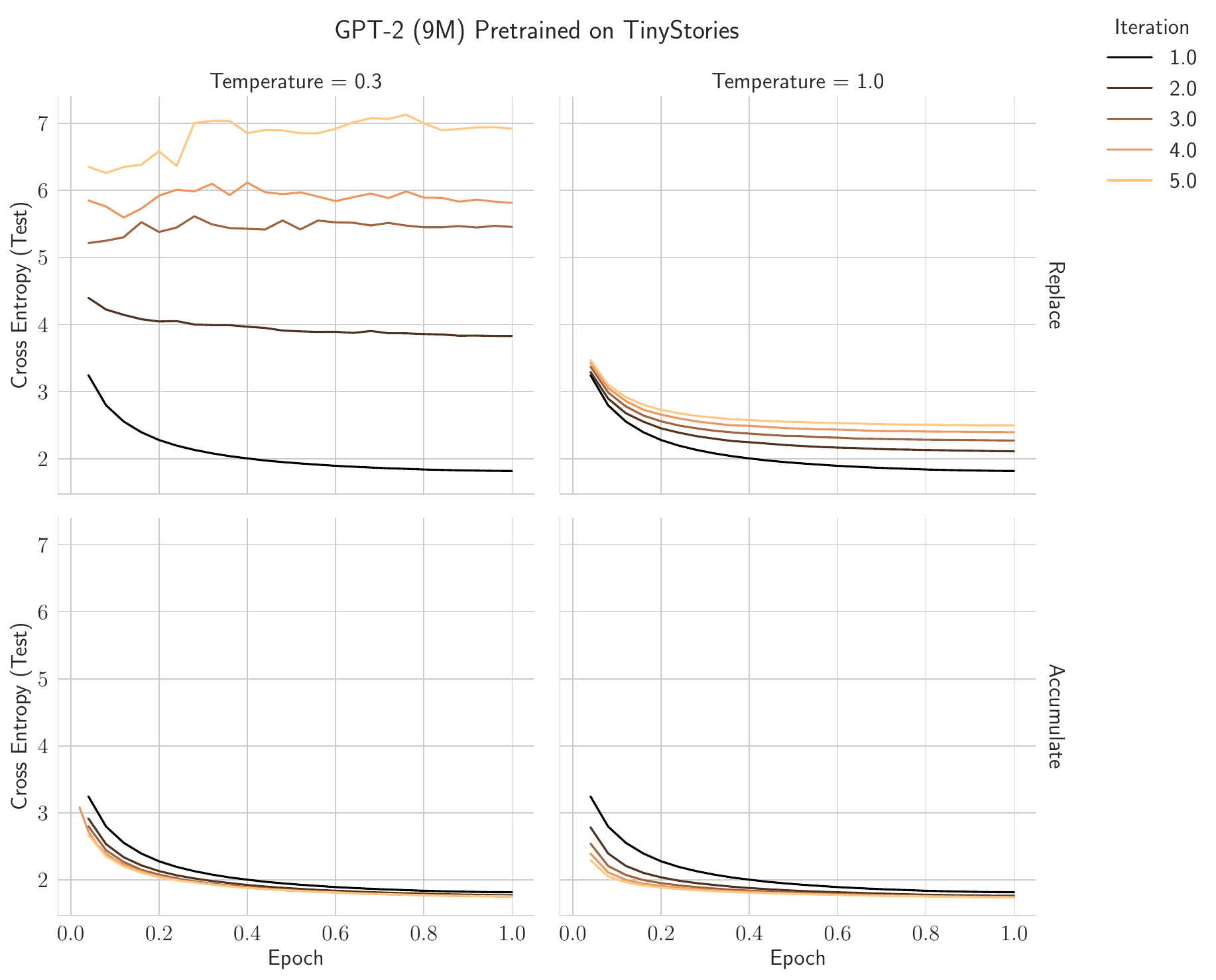}
    \caption{Accumulating data shows stable behavior across different generation temperatures for a GPT-2 (9M) model, while replacing data does not.
    }
    \label{fig:language_modeling_temp}
\end{figure}

\begin{figure}
    \centering
    \includegraphics[width=\textwidth]{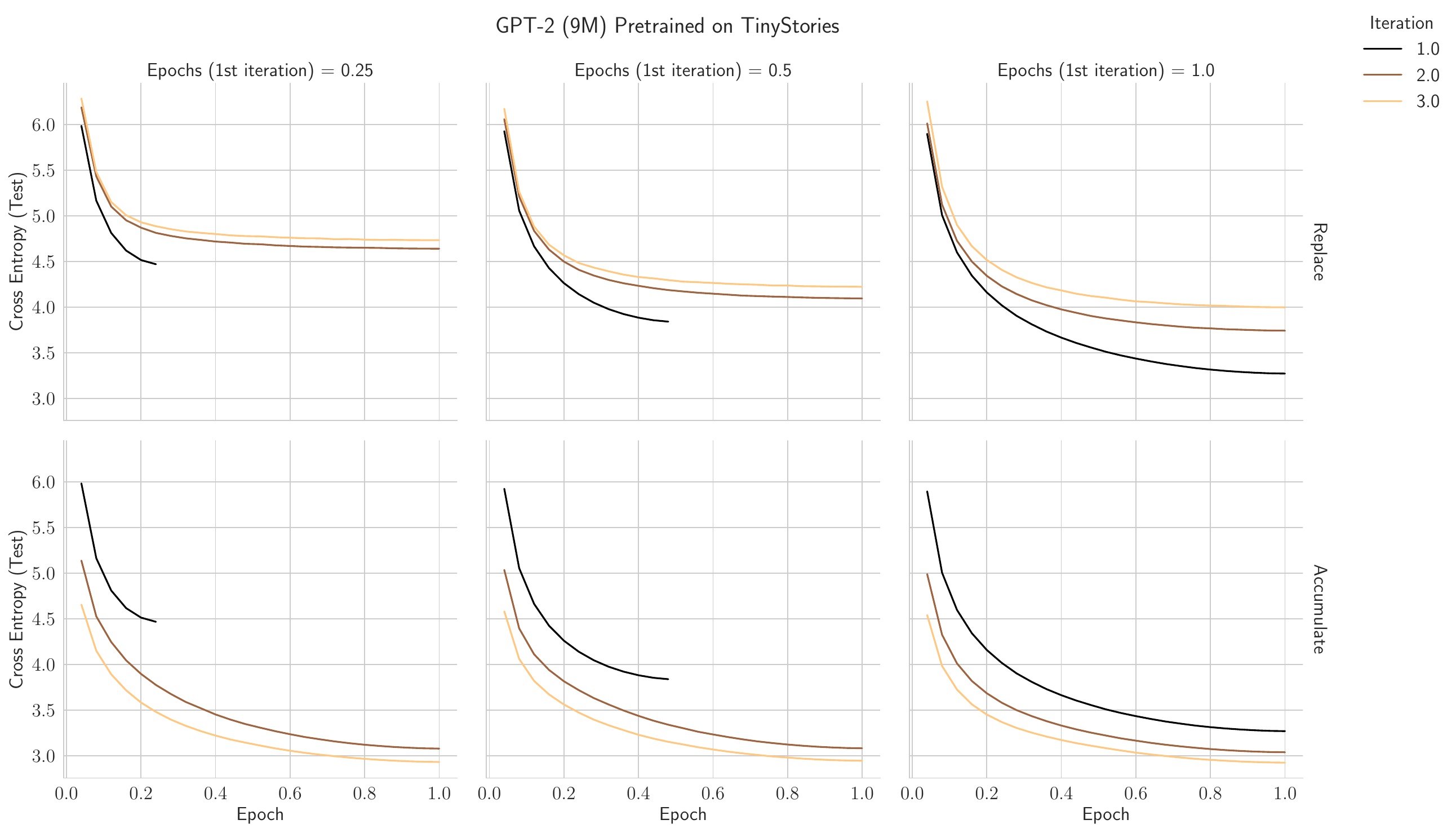}
    \caption{Model quality after the first model-fitting iteration does not qualitatively change behavior in subsequent iterations. Columns show differing training amount (as measure by epochs) in first iteration.
    }
    \label{fig:language_modeling_firstiter}
\end{figure}

\clearpage
\section{Additional Details on VAE Experiments}
\label{app:sec:vae_ablations}
\paragraph{Experiment Details.}
As pre-processing, we crop and down-sample the images to 64x64 pixels. We use a standard convolutional architecture for the VAE model consisting of 5 convolutional layers with 32, 64, 128, 256, and 512 channels, respectively, and a similar convolutional decoder structure. The latent space is 128-dimensional isotropic Gaussian, represented by 2 MLP layers. Each data iteration consists of 100 training epochs, after which we generate 163K new training images by sampling latents from the Gaussian prior and the passing them through the generator model.

\paragraph{Analysis of Reconstructions.}
Figure~\ref{fig:vae_reconstructions} shows reconstructions after replacing (left) and accumulating (center) data, compared to baseline (right).
Analyzing the reconstruction of test set images also reveals interesting findings - the model trained only on data from the prior iteration has indeed collapsed and cannot represent any other classes besides the single mode it generates. Interestingly, the model trained on aggregated data still maintains it's capabilities and generates accurate reconstructions, including smaller details such as glasses and hats. We hypothesize that this model maintains it's generative capabilities, but these details become a more minor sub-manifold in the latent space, which is realigned with the newly-generated data, hence why they appear less often in the generated images, which use samples from the prior.

\begin{figure}[h!]
    \centering
    \includegraphics[width=0.325\textwidth]{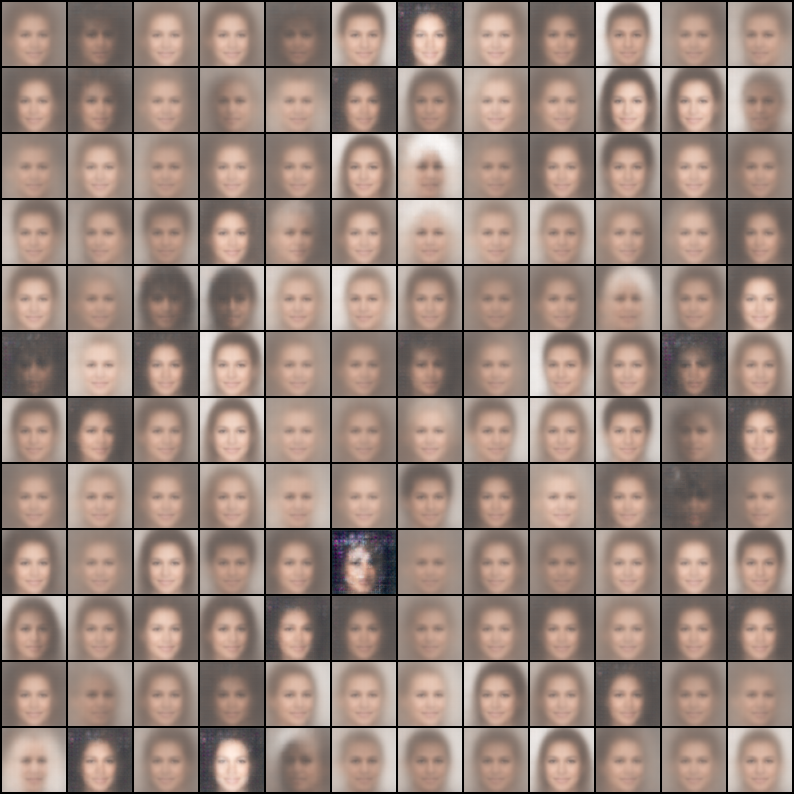}
    \includegraphics[width=0.325\textwidth]{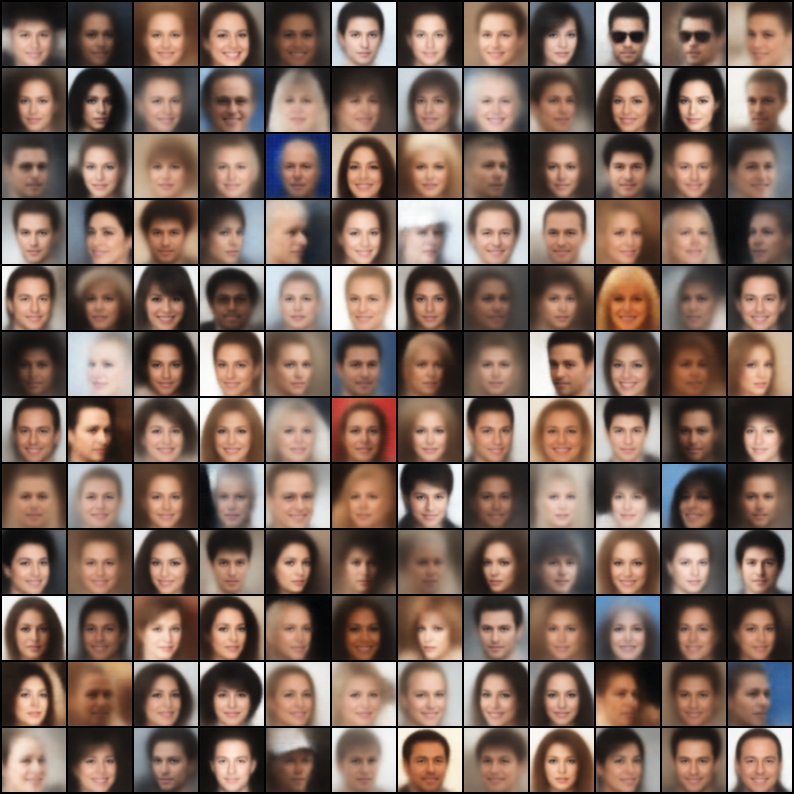}
    \includegraphics[width=0.325\textwidth]{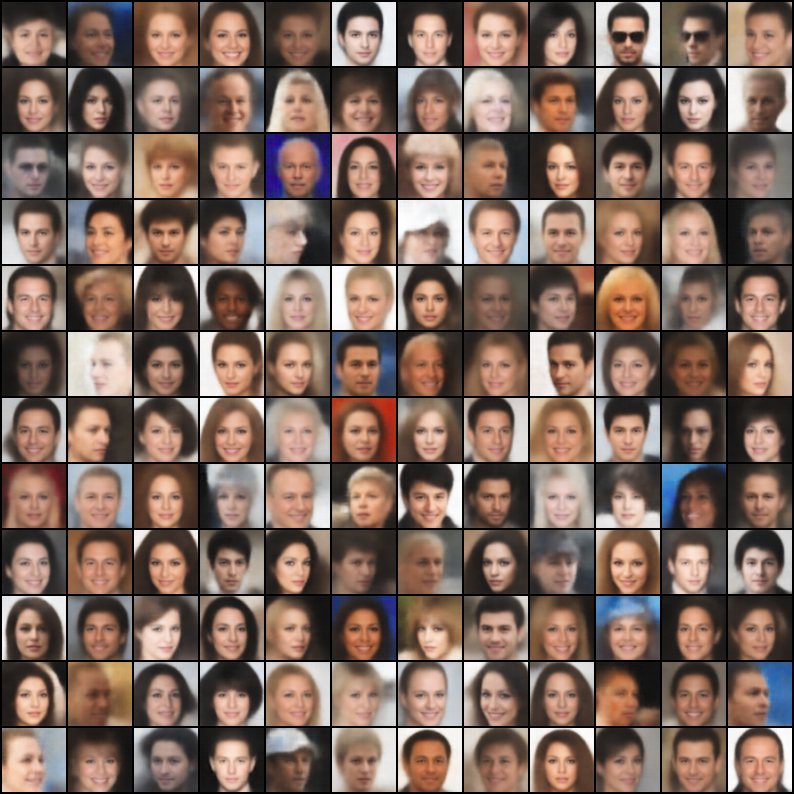}

    \caption{\textbf{Data Accumulation Maintains Model Capabilities.} Image reconstructions from the test set. Left: Training on prior iterations collapses the model's capability, and subsequently, it can only represent a single mode. Middle: training on aggregated data preserves model capabilities and leads to little to no degradation in the reconstructed images. Right: Baseline reconstructions after 100 training epochs on the dataset.}
    \label{fig:vae_reconstructions}
\end{figure}

\clearpage
\section{Linear Regression: Replacing Data with Increasing Sample Size}
\label{app:sec:linear_reg_replace_multiple}

In the framework of \citet{mobahi2020self} and \citet{dohmatob2024model}, we consider sequences of linear models fit to the previous model's synthetic outputs. Within this framework, \citet{dohmatob2024model} proved that \replace{if data are replaced} with each model fitting iteration and the training data cardinality remains constant, \replace{then the test squared error scales linearly with the number of model fitting iterations $n$}:

\begin{equation}
E_{\text{test}}^{\text{Replace}}(\hat{w}_n) = \frac{\sigma^2 d}{T - d -1} \times \color{red}n \color{black}
\end{equation}

In this work, we lightly adapt the argument of \citet{dohmatob2024model} to study the effects \accumulate{if data accumulate} with each model fitting iteration. We specifically considered the case where the training data cardinality increases by a constant $T$ with each model-fitting iteration i.e. the $i$th model is fit using $T\times i$ data, where $T$ data are ``real" and then each subsequently fit model contributes its own $T$ synthetic data to the accumulating data. In this setting, \accumulate{the test squared error is upper bounded independent of the number of iterations}.

\begin{equation}
E_{\text{test}}^{\text{Accumulate}}(\hat{w}_n) = \frac{\sigma^2 d}{T - d -1} \times \sum_{k=1}^n \dfrac{1}{k^2} \leq \frac{\sigma^2 d}{T - d -1} \times \color{red}\frac{\pi^2}{6} \color{black}
\end{equation}

In the main text, we focus on the replace and accumulate data settings because prior work focused on replacing data and we wished to study how accumulating data affects model collapse. However, a much richer landscape of outcomes is possible. For instance, and as pointed out in personal correspondence with \citet{dohmatob2024model}, one can consider what we term the \replacemultiple{``Replace-Multiple"} setting, in which one fits the $i$-th linear model using $T\times i$ data sampled from the $(i-1)$-th linear model. Replace-Multiple is a useful baseline for Accumulate because it matches the amount of training data at each model fitting iteration. Under \replacemultiple{Replace-Multiple, the test squared error grows logarithmically}:

\begin{equation}
E_{\text{test}}^{\text{Replace-Multiple}}(\hat{w}_n) = \frac{\sigma^2 d}{T - d -1} \times \sum_{k=1}^n \dfrac{1}{k} \approx \frac{\sigma^2 d}{T - d -1} \times \color{red} \log(n) \color{black}
\end{equation}

Replace-Multiple has the drawback of not matching the total amount of compute of Accumulate since each iteration of Replace-Multiple draws $T\times i$ samples from the most recent model, whereas Accumulate draws $T$ samples from the most recent model. 
Other baselines are also possible, but we leave these to future work. We focus on accumulating data as we feel real and synthetic data are likely to accumulate in the real world as time progresses.

\clearpage
\section{Additional Linear Regression Numerical Results}
\label{app:sec:linear_regression_numerics}

\begin{figure}[h!]
    \centering
    \includegraphics[width=\textwidth]{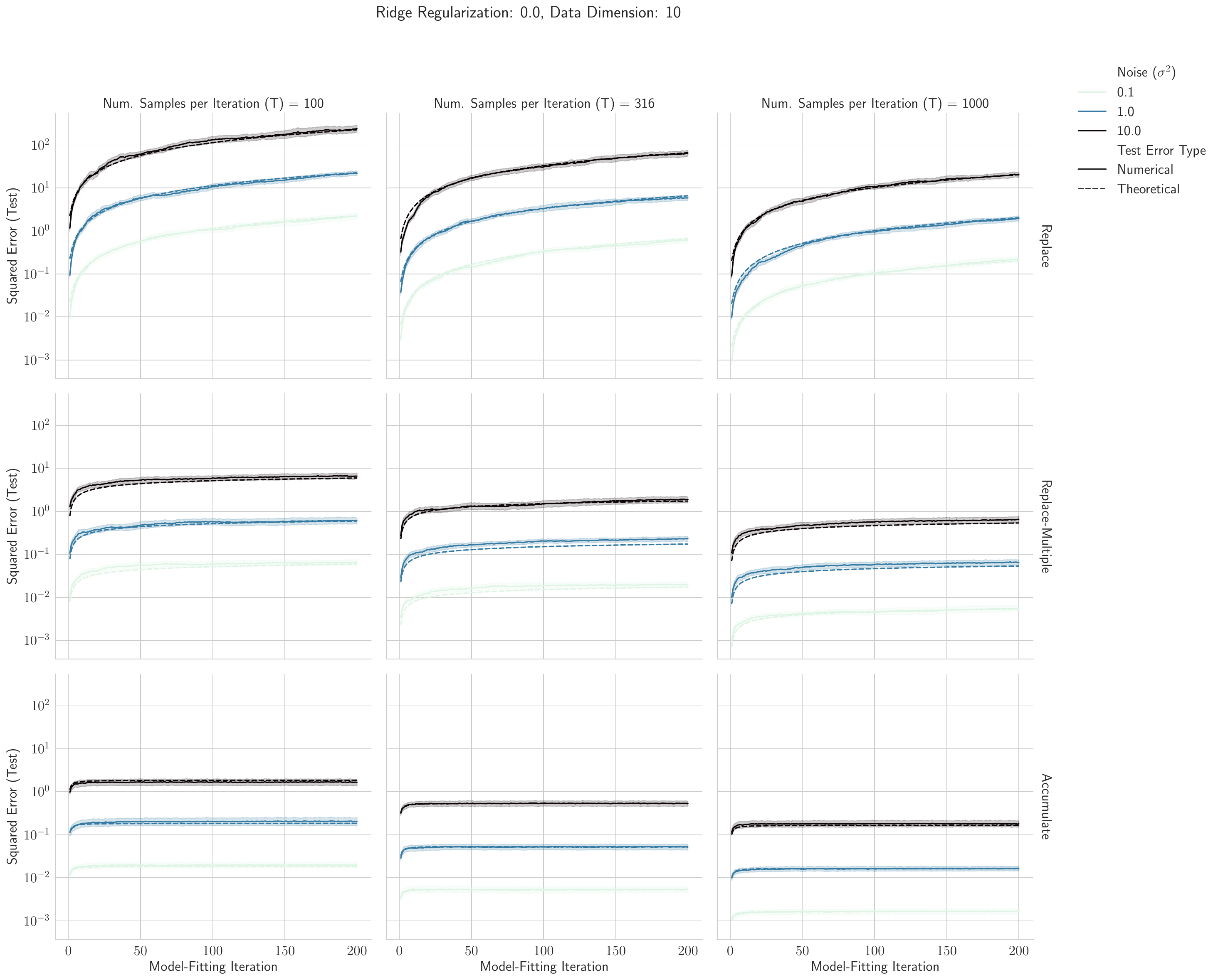}
    \caption{\textbf{Accumulating data across iterations avoids model collapse in linear regression.} We consider sequences of linear models recurrently fit to generated targets by previous iterations of models. \replace{Replace} (Top): If each linear model is fit to the generated targets of \textit{only} the preceding linear model i.e. data are replaced, then the \replace{test squared error grows linearly} with the number of model-fitting iterations iterations $n$. \replacemultiple{Replace-Multiple} (Middle): If each linear model is fit to $T\times i$ samples from the $(i-1)$-th model (i.e. the same amount of data as Accumulate), then the \replacemultiple{test squared error grows logarithmically} with the number of model-fitting iterations; see Appendix \ref{app:sec:linear_reg_replace_multiple} for more details. \accumulate{Accumulate} (Bottom): If each linear model is instead fit to the generate targets of \textit{all} the preceding linear models i.e. data accumulate, then the \accumulate{test squared error has a finite upper bound, independent of the number of iterations}. 
    This suggests that data accumulation might be a robust solution for mitigating model collapse. This figure is similar to Figure \ref{fig:accumulating_vs_nonaccumulating_isotropic_features} but displaying \textit{log} test squared error and more model-fitting iterations for additional clarity.}
    \label{app:fig:accumulating_vs_nonaccumulating_isotropic_features_log_linear}
\end{figure}

\newpage

\end{document}